%% file: main.tex
\documentclass{article}

\PassOptionsToPackage{numbers, compress}{natbib}

\usepackage[preprint]{neurips_2026}

\usepackage[utf8]{inputenc} 
\usepackage[T1]{fontenc}    
\usepackage{hyperref}       
\usepackage{url}            
\usepackage{booktabs}       
\usepackage{amsfonts}       
\usepackage{nicefrac}       
\usepackage{microtype}      
\usepackage{xcolor}         
\usepackage{amsmath}
\usepackage{comment}
\usepackage{multirow}
\usepackage{colortbl}
\usepackage{pdfpages}
\usepackage{bm}
\usepackage{cleveref}
\usepackage{subcaption}
\usepackage{float}
\usepackage{nicematrix}       
\usepackage{tikz}

\crefname{figure}{Fig.}{Figs.}
\crefname{table}{Tab.}{Tabs.}
\crefname{equation}{Eq.}{Eqs.}
\crefname{section}{Sec.}{Secs.}
\crefname{subsection}{Sec.}{Secs.}
\crefname{subsubsection}{Sec.}{Secs.}
\crefname{appendix}{App.}{Apps.}
\crefname{subappendix}{App.}{Apps.}
\crefname{subsubappendix}{App.}{Apps.}
\crefname{align}{Eq.}{Eqs.}
\title{DPPE: Rethinking Camera-Based Positional Encoding for Scaling Multi-View Transformers}

\author{%
  Shun Kenney\thanks{Work was done during an internship at SB Intuitions.} \\
  Keio University\\
  Yokohama, Japan \\
  \texttt{shunkenney@keio.jp} \\
  \And
  Teppei Suzuki \\
  SB Intuitions Corp. \\
  Tokyo, Japan \\
  \texttt{teppei.suzuki@sbintuitions.co.jp} \\
}

\begin{document}

\maketitle

\input{sections/00_abstract}

\input{sections/01_intro}

\input{sections/02_preliminaries}

\input{sections/03_analysis_prope}
\input{sections/04_method}
\input{sections/05_experiments}
\input{sections/07_conclusion}

\bibliographystyle{unsrtnat}
\bibliography{references}

\appendix
\crefalias{section}{appendix}
\crefalias{subsection}{subappendix}
\crefalias{subsubsection}{subsubappendix}

\input{sections/08_appendix}

\end{document}

%% file: sections/00_abstract.tex
\begin{abstract}
The remarkable scalability of Transformers has expanded their application to 3D computer vision, where camera-aware positional encoding is crucial for providing spatial cues in multi-view geometry.
Recent advancements have established the practice of using camera parameters---such as extrinsics or projection matrices---as relative positional encoding into the query, key, and value vectors of the attention mechanism.
However, when scaling up the training recipe of novel view synthesis (NVS) models with the camera-based positional encoding, we observe a significant issue: model performance stagnates in the late stages of training.

In this paper, we investigate the cause of the performance bottleneck when scaling up and demonstrate that 
storing rotation and translation given by the positional encoding in the same dimensions of the value vector causes indeterminacy in their independent identification, hindering training scalability.
To address this, we propose Decoupled Pose Positional Encoding (DPPE), a novel camera-based positional encoding that explicitly decouples rotation and translation.
Extensive evaluations on NVS tasks demonstrate that DPPE enables stable long-term training even in scaled-up training setup.
Furthermore, it exhibits superior generalization performance in extrapolation settings, such as handling an increased number of viewpoints and zoom-in scenarios.
Code is available at \url{https://github.com/sbintuitions/dppe}.
\end{abstract}

%% file: sections/01_intro.tex
\section{Introduction}
The remarkable scalability of Transformers~\cite{transformer} has led to their dominance not only in natural language processing but also in numerous other areas, including the 3D computer vision field~\cite{lvsm,prope,vggt,ma,da3}.
The attention function~\cite{transformer}, which governs the information exchange between tokens in the Transformer, is implemented with positional encoding, which explicitly injects positional information into tokens for tasks where sequence order matters, since the attention function is itself permutation-invariant.
There are various positional encoding methods~\cite{transformer,prope,learnablepe,cpe,alibi,rope,ropevit,circlerope,ganpe,llmpe1,llmpe2,llmpe3,spiralrope,drope,cape,gta,rayrope,ucpe}, and in recent years, RoPE~\cite{rope} is widely adopted because of its valuable properties such as the flexibility of sequence length and the capability of equipping the linear self-attention with relative position encoding.

For the tasks depending on multi-view geometry, camera-based positional encodings~\cite{cape,gta,prope,ucpe,urope,rayrope,redirector} are employed to provide spatial cues to tokens.
In contrast to pixel-aligned approaches, which adds Pl\"ucker rays or raymaps to RGB~\cite{lvsm,lgm,gslrm,grm,cuter,cat3d}, several methods~\cite{prope,cape,gta,ucpe,redirector,rayrope,urope} have been proposed that apply camera-based transformations (\textit{e.g.}, camera extrinsics or projection matrices) as relative positional encoding within attention.
Notably, GTA~\cite{gta} aligned the coordinate systems of tokens in the feature space by extending the camera-based geometric transformations to the values alongside the queries and keys.
This geometry-aware mechanism enables GTA to significantly improve both the learning efficiency and overall performance of 3D vision tasks without requiring any additional learnable parameters, and the introduction of positional encoding to the values is widely used in several methods~\cite{prope,ucpe,redirector,rayrope,urope}.

\begin{figure}[t]
  \begin{center}
    \begin{subfigure}{0.55\textwidth}
      \centering
      \includegraphics[width=\textwidth]{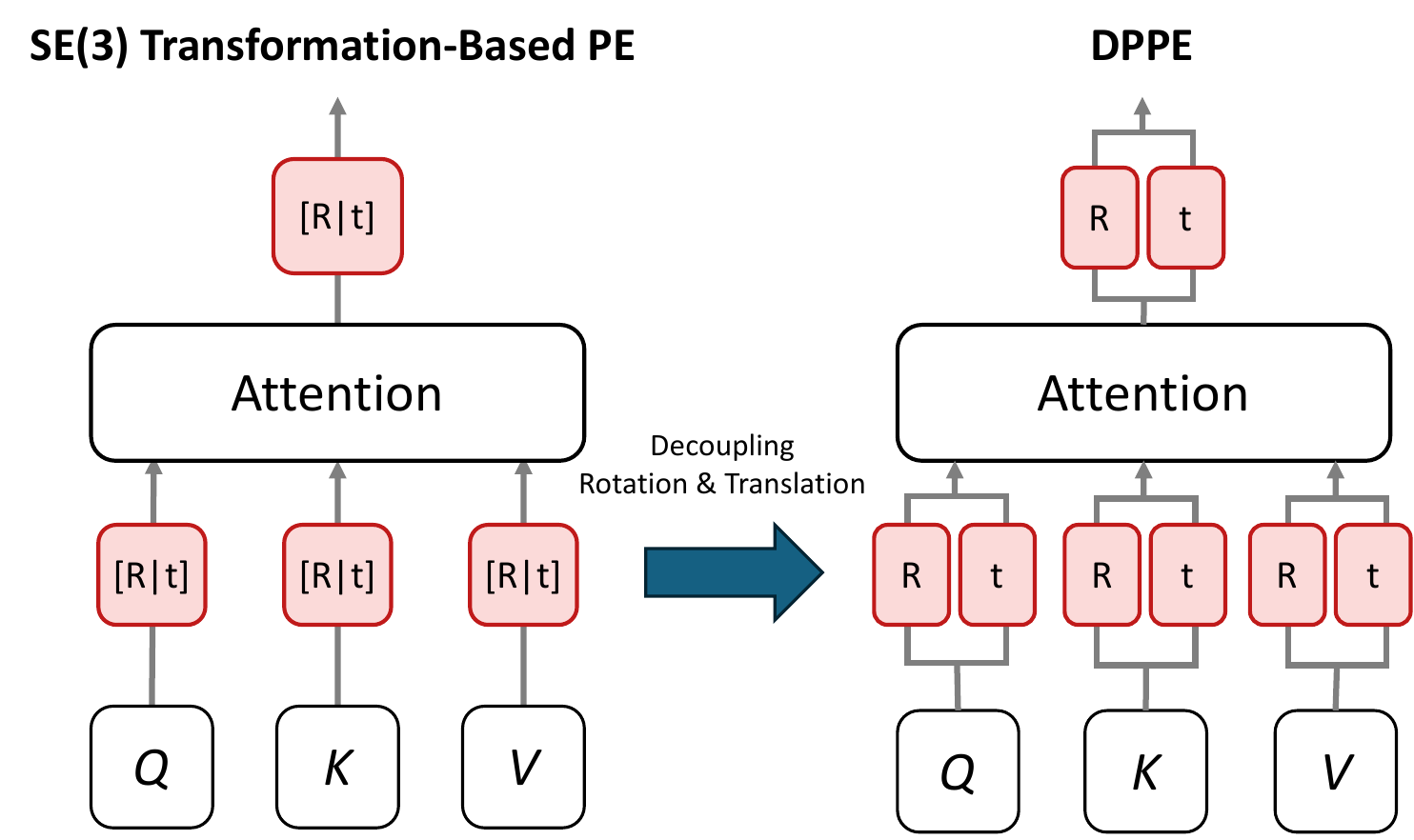}
      \subcaption{\textbf{Methods comparison}}
      \label{fig:pe_comparison_a}
    \end{subfigure}
    \hspace{0.02\textwidth}
    \begin{subfigure}{0.35\textwidth}
      \centering
      \includegraphics[width=\textwidth]{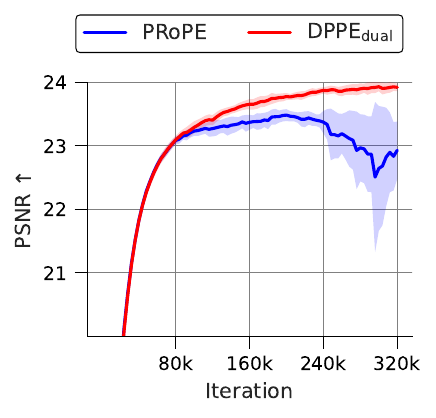}
      \subcaption{\textbf{NVS results on MVImgNet2}}
      \label{fig:pe_comparison_b}
    \end{subfigure}
  \end{center}
  \vspace*{-4pt}
  \caption{\textbf{DPPE prevents the late-stage performance degradation observed in PRoPE.} (a) Conceptual comparison between existing projection matrix–based positional encoding (left) and our proposed DPPE (right). Existing methods apply a single coupled SE(3) transformation $[R\,|\,t]$ uniformly to the queries (Q), keys (K), and values (V) before attention. In contrast, DPPE, our proposed method, decouples the rotation $R$ and translation $t$, and injects them into Q, K, and V according to their distinct geometric roles. (b) Novel view synthesis on MVImgNet2: Results represent the mean over three seeds, and the shaded regions indicate the standard deviation. PRoPE's PSNR stagnates and degrades after 240k iterations, whereas DPPE remains stable.}
    \label{fig:pe_comparison}
\end{figure}

Despite the advantage of camera-based positional encoding and its empirically verified effectiveness, our validation on the novel view synthesis (NVS) task using a scaled-up setting, \textit{i.e.}, a larger model trained on more data for an extended number of iterations, revealed a limitation: the performance of PRoPE, a natural extension of GTA, stagnates in the late stages of training (see \cref{fig:pe_comparison_b}).

To investigate the cause of this performance stagnation, we systematically isolate the distinct factors within PRoPE~\cite{prope}, which is a projection matrix-based positional encoding, and analyze their specific impacts on the NVS task that strictly requires 3D spatial understanding.
From those experiments, we obtain the following insight into the incorporation of camera information into the values in the attention function: \textit{In camera-based positional encoding for the values, performance degrades if the rotation and shift components of the camera are coupled within the same dimensions.}

Based on these findings, we propose Decoupled Pose Positional Encoding (DPPE), a camera-based positional encoding in which rotation and translation are decoupled, as shown in \cref{fig:pe_comparison_a}.
The proposed method enables stable training even in scaled-up configurations and demonstrates better generalization than baseline methods, including robustness to viewpoints beyond the training distribution and to zoom-in scenarios.

%% file: sections/02_preliminaries.tex
\section{Related Work}
\label{sec:related_work}
\paragraph{Positional Encoding in Transformers.}
Absolute positional encoding (APE) directly injects per-token positional information, originating with the sinusoidal embeddings of \cite{transformer} and extended via learned and conditional variants ~\cite{cpe,ganpe}.
Relative positional encoding (RPE) instead encodes pairwise offsets~\cite{learnablepe,llmpe1,llmpe2,llmpe3,alibi}, among which Rotary Position Embedding (RoPE)~\cite{rope} has become predominant due to its flexibility with variable sequence lengths and compatibility with linear attention, and has been adapted to vision and multimodal settings~\cite{ropevit,circlerope,spiralrope,drope}.

\paragraph{Camera-Based Positional Encoding for Multi-View Transformers.}
Extending Transformers to multi-view 3D tasks that require explicit reasoning over camera geometry demands positional encodings capable of conveying such geometry.
Existing methods fall into two broad categories.
Pixel-aligned approaches concatenate Pl\"ucker rays or raymaps to the input RGB, providing per-pixel geometric cues without modifying attention; this strategy can be viewed as an APE in the multi-view geometry setting~\cite{lvsm,lgm,gslrm,grm,cuter,cat3d}.
Attention-based approaches integrate camera transformations directly into attention as RPEs, and can be further classified by where in the attention mechanism the encoding is applied: (i) QK-only methods such as CaPE~\cite{cape} inject camera information only into queries and keys, biasing attention weights without altering value features; (ii) QKV methods including GTA~\cite{gta}, PRoPE~\cite{prope}, and several recent extensions~\cite{ucpe,redirector} additionally transform values and outputs, aligning feature content across coordinate frames.
Concurrent works also explore ray-based and unified relative formulations~\cite{rayrope,urope}.

\paragraph{Scope of This Work.}
GTA~\cite{gta} introduced the principle of extending camera-based positional encoding from the queries and keys to the values and outputs (vo-pe), aligning feature coordinate systems across views.
PRoPE~\cite{prope} generalizes this to projection matrices and applies the same encoding function to both query-key positional encoding (qk-pe) and value-output positional encoding (vo-pe), as do subsequent works in this family.
Despite their functionally distinct roles---qk-pe shaping attention similarity, vo-pe transforming aggregated feature content---this uniform treatment has not been systematically examined, particularly under scaled-up training.
This paper investigates that question, focusing on the GTA/PRoPE family.
We adopt PRoPE~\cite{prope} as the exclusive baseline because (i) PRoPE generalizes GTA (the two coincide under fixed intrinsics, as confirmed in ~\cite{prope}), so findings on PRoPE extend to GTA; (ii) RayRoPE~\cite{rayrope} introduces a learnable depth predictor and addresses ray-based geometry rather than qk-pe/vo-pe uniformity; (iii) methods for generative or video-retake tasks~\cite{ucpe,redirector} operate under fundamentally different paradigms.

\section{Preliminaries}
\label{sec:preliminary}
Here, we review the formulation of PRoPE, our baseline method.
Given a projection matrix $\bm{P}_i \in \mathbb{R}^{4 \times 4}$ for the viewpoint of the image associated with the $i$-th token, the embedding matrix of the projective positional encoding in PRoPE is described as follows:
\begin{align}
    \gamma^{\mathrm{Proj}}_{d/4}(i) = \bm{I}_{d/4} \otimes \bm{P}_i, \label{eq:prope}
\end{align}
where $\otimes$ denote the Kronecker product and $\bm{I}_k \in \{0,1\}^{k \times k}$ be the identity matrix, and $\gamma^{\mathrm{Proj}}_{d/4}(i)\in\mathbb{R}^{d\times d}$ is a block diagonal matrix.
For clarity, the dimension of the identity matrix is indicated in the subscript of $\gamma^{\mathrm{Proj}}(i)$, but we omit it hereafter unless necessary.

Given the $d$ dimensional queries, keys, and values, $\bm{q}, \bm{k}, \bm{v} \in \mathbb{R}^{N \times d}$,
where $N$ is the number of tokens, the self-attention with the projective positional encoding for the $i$-th query is defined as follows:
\begin{align}
    \label{eq:prope_attn}
    \bm{o}'_i &= \gamma^{\mathrm{Proj}}(i) \, \mathrm{SelfAttn}(\bm{q}, \bm{k}, \bm{v}, \gamma^{\mathrm{Proj}}) \notag\\
    &= \gamma^{\mathrm{Proj}}(i) \sum_{j} \alpha_{ij} \, \gamma^{\mathrm{Proj}}(j)^{-1} \bm{v}_j,
\end{align}
where the attention weight $\alpha_{ij}$ is defined as:
\begin{align}
    \label{eq:softmax}
    \alpha_{ij} = \frac{\exp\!\left(s_{ij} / \sqrt{d}\right)}{\sum_{l} \exp\!\left(s_{il} / \sqrt{d}\right)},
    \quad
    s_{ij} = \bm{q}_i^{\top} \gamma^{\mathrm{Proj}}(i) \, \gamma^{\mathrm{Proj}}(j)^{-1} \bm{k}_j.
\end{align}

Since $\gamma^{\mathrm{Proj}}(i)$ applied to the output of the self-attention is independent of the index $j$, it can be moved inside the summation and combined with $\gamma^{\mathrm{Proj}}(j)^{-1}$ applied to $\bm{v}_j$.
This allows them to be computed together in advance as ${\gamma^\mathrm{Proj}}(i){\gamma^\mathrm{Proj}}(j)^{-1}=\bm{I}_{d/4} \otimes\bm{P}_i {\bm{P}_j}^{-1}$, which serves as a coordinate system alignment~\cite{gta,prope}.
This role differs fundamentally from that of the positional encoding for query-key pairs, whose primary purpose is to inject biases reflecting token similarity.

PRoPE is used in conjunction with RoPE~\cite{rope} as $\gamma^{\mathrm{PRoPE}} = \left[\begin{smallmatrix} \gamma^\mathrm{Proj}_{d/8} & 0 \\ 0 & \gamma^\mathrm{RoPE}_{d/4}(x,y) \end{smallmatrix}\right]\in\mathbb{R}^{d\times d}$, where $\gamma^\mathrm{RoPE}_{d/4}(x,y)\in\mathbb{R}^{\frac{d}{2}\times \frac{d}{2}}$ is the rotary embeddings based on the 2D patch coordinates $x$ and $y$, \textit{i.e.}, $\gamma^\mathrm{RoPE}_{d/4}(x,y)=\left[\begin{smallmatrix}
\gamma^{\mathrm{RoPE}}_{d/4}(x) & 0 \\ 0 & \gamma^{\mathrm{RoPE}}_{d/4}(y)
\end{smallmatrix}\right]$, where $\gamma^{\mathrm{RoPE}}_{d/4}(\cdot)$ constructs $\frac{d}{4}\times\frac{d}{4}$ rotary embeddings.
We can define the original formulation of PRoPE by replacing $\gamma^{\mathrm{Proj}}$ in \cref{eq:prope_attn} with $\gamma^{\mathrm{PRoPE}}$.

%% file: sections/03_analysis_prope.tex
\section{Investigating the Bottleneck of Camera-Based Positional Encoding}
\label{sec:explor_prope}
In this section, we analyze the functionality of PRoPE~\cite{prope} using the novel view synthesis (NVS) task.
We decompose the positional information constituting PRoPE into 2D patch coordinates, the camera intrinsics, and camera extrinsics to determine which parameters are responsible for the performance stagnation observed under scaled-up training.

\paragraph{Training Setup}
We evaluate all combinations of the parameters across three datasets: MVImgNet2~\cite{mvimgnet2}, RealEstate10K~\cite{re10k}, and SpatialVidHQ~\cite{spatialvid}, using standard RGB inputs.
To investigate the effect of additional geometric input, we conduct an extended study on MVImgNet2 by concatenating Pl\"ucker raymaps to the input RGB in addition to a simple RGB input.
For the model, we adopt the large-scale training configuration used in PRoPE~\cite{prope}, \textit{i.e.}, a Transformer with 12 layers, a hidden dimension of 768, and an MLP dimension of 3072.
We scale up the number of training iterations to 320k to accommodate the scaled-up dataset.
The number of input views is set to two views to simplify the analysis.
Detailed experimental settings are provided in \cref{app:exp_prope_detail}.

\subsection{What is the Bottleneck in PRoPE?}
In PRoPE, embeddings of 2D patch positions and the projection matrices are applied to both query-key and value-output pairs in self-attention.
As a result, the positional encoding affects the attention pattern and the aggregated output, yet the precise role of each component within each encoding has not been thoroughly examined.

Therefore, to clarify the distinct properties of QK and VO positional encodings (qk-pe and vo-pe) and to assess the individual contributions of 2D patch coordinate and camera information, we evaluate all well-posed combinations of these encodings.

\paragraph{Results}

\begin{table}[t]
  \setlength{\fboxsep}{0pt}
  \caption{\textbf{Quantitative results of various positional encoding combinations on various datasets.} ``$\bm{P}$'' and ``2D'' denote the projective positional encoding and the 2D patch positional encoding, respectively. For each metric, the best, second-best, and third-best results are highlighted in \colorbox{red!40}{red}, \colorbox{orange!30}{orange}, and \colorbox{yellow!100!black!30}{yellow}, respectively. We report PSNR$\uparrow$/SSIM$\uparrow$/LPIPS$\downarrow$.}
  \label{tab:all_datasets_id_sweep}
  \centering
  \setlength{\tabcolsep}{4pt} 

  \vspace*{-8pt}
  
  \begin{subtable}{\linewidth}
    \captionsetup{aboveskip=0pt}
    \centering
    \caption{MVImgNet2}
    \label{tab:mvimg_id_sweep}
    \begin{tabular}{@{}lcccc@{}}
      \toprule
      \multirow{2}{*}{qk-pe} & \multicolumn{4}{c}{vo-pe} \\
      \cmidrule(l){2-5}
      & None & 2D & $\bm{P}$ & 2D+$\bm{P}$ \\
      \cmidrule(r){1-1} \cmidrule(lr){2-2} \cmidrule(lr){3-3} \cmidrule(lr){4-4} \cmidrule(l){5-5}
      None & - & - & - & 23.10 / 0.699 / 0.211 \\
      2D   & - & - & 23.90 / 0.728 / {0.181} & {23.50} / 0.713 / 0.194 \\
      $\bm{P}$    & - & \colorbox{orange!30}{24.21} / \colorbox{red!40}{0.739} / \colorbox{yellow!100!black!30}{0.174} & - & {23.98} / 0.731 / 0.183 \\
      2D+$\bm{P}$ & \colorbox{yellow!100!black!30}{24.09} / \colorbox{yellow!100!black!30}{0.733} / \colorbox{orange!30}{0.172} & \colorbox{red!40}{24.22} / \colorbox{orange!30}{0.738} / \colorbox{red!40}{0.171} & 23.81 / 0.725 / {0.184} & {23.73} / {0.725} / {0.183} \\
      \bottomrule
    \end{tabular}
  \end{subtable}
  \begin{subtable}{\linewidth}
    \centering
    \caption{RealEstate10K}
    \label{tab:re10k_id_sweep}
    \begin{tabular}{@{}lcccc@{}}
      \toprule
      \multirow{2}{*}{qk-pe} & \multicolumn{4}{c}{vo-pe} \\
      \cmidrule(l){2-5}
      & None & 2D & $\bm{P}$ & 2D+$\bm{P}$ \\
      \cmidrule(r){1-1} \cmidrule(lr){2-2} \cmidrule(lr){3-3} \cmidrule(lr){4-4} \cmidrule(l){5-5}
      None & - & - & - & {25.94} / {0.891} / 0.099 \\
      2D   & - & - & {26.01} / {0.892} / 0.099 & \colorbox{orange!30}{26.20} / \colorbox{yellow!100!black!30}{0.894} / \colorbox{yellow!100!black!30}{0.097} \\
      $\bm{P}$    & - & {26.11} / {0.893} / \colorbox{yellow!100!black!30}{0.097} & - & {25.93} / {0.892} / 0.098 \\
      2D+$\bm{P}$ & {26.15} / {0.891} / 0.098 & \colorbox{red!40}{26.24} / \colorbox{red!40}{0.898} / \colorbox{red!40}{0.094} & {26.06} / {0.892} / 0.098 & \colorbox{yellow!100!black!30}{26.18} / \colorbox{orange!30}{0.895} / \colorbox{orange!30}{0.095} \\
      \bottomrule
    \end{tabular}
  \end{subtable}
  \begin{subtable}{\linewidth}
    \centering
    \caption{SpatialVidHQ}
    \label{tab:spavid_id_sweep}
    \begin{tabular}{@{}lcccc@{}}
      \toprule
      \multirow{2}{*}{qk-pe} & \multicolumn{4}{c}{vo-pe} \\
      \cmidrule(l){2-5}
      & None & 2D & $\bm{P}$ & 2D+$\bm{P}$ \\
      \cmidrule(r){1-1} \cmidrule(lr){2-2} \cmidrule(lr){3-3} \cmidrule(lr){4-4} \cmidrule(l){5-5}
      None & - & - & - & \colorbox{yellow!100!black!30}{19.65} / \colorbox{orange!30}{0.649} / \colorbox{yellow!100!black!30}{0.248} \\
      2D   & - & - & {19.62} / {0.646} / 0.249 & \colorbox{red!40}{19.68} / \colorbox{orange!30}{0.649} / \colorbox{red!40}{0.246} \\
      $\bm{P}$    & - & {19.58} / {0.647} / 0.251 & - & \colorbox{orange!30}{19.66} / \colorbox{red!40}{0.651} / 0.249 \\
      2D+$\bm{P}$ & {19.58} / {0.645} / \colorbox{yellow!100!black!30}{0.248} & {19.58} / {0.644} / 0.250 & {19.64} / {0.647} / \colorbox{orange!30}{0.247} & \colorbox{yellow!100!black!30}{19.65} / {0.648} / {0.251} \\
      \bottomrule
    \end{tabular}
  \end{subtable}

\end{table}

We show PSNR, SSIM~\cite{ssim}, and VGG~\cite{simonyan2014very}-based LPIPS~\cite{lpips} of each combination on each dataset in \cref{tab:all_datasets_id_sweep}.
For RealEstate10K and SpatialVidHQ, no significant performance differences are observed across the various combinations of positional encodings.
Similarly, on MVImgNet2, the choice of QK positional encoding does not yield substantial differences.
In contrast, for the value-output positional encoding on MVImgNet2, we observe performance improvements when using the ``2D'' or ``None'' settings.

This result suggests that, on certain datasets, the positional encoding using projection matrix for the value-output pair negatively impacts performance.
A similar trend is also observed in \cref{tab:mvimg_plucker_id_sweep}, which shows the metrics when Pl\"ucker raymaps are concatenated to the RGB input for MVImgNet2.
This explicitly rules out the absence of geometric input as the cause of degradation, indicating that applying projection-matrix-based positional encoding to the value-output pair is itself harmful. 
We investigate the underlying cause in the next section.

\begin{table}[t]
  \caption{\textbf{Quantitative results for various positional encoding combinations with Pl\"ucker raymaps on MVImgNet2.} Notations and metrics follow \cref{tab:all_datasets_id_sweep}.}
  \label{tab:mvimg_plucker_id_sweep}
  \centering
  \setlength{\tabcolsep}{4pt}
  \setlength{\fboxsep}{0pt}
  \begin{tabular}{@{}lcccc@{}}
    \toprule
    \multirow{2}{*}{qk-pe} & \multicolumn{4}{c}{vo-pe} \\
    \cmidrule(l){2-5}
    & None & 2D & $\bm{P}$ & 2D+$\bm{P}$ \\
    \cmidrule(r){1-1} \cmidrule(lr){2-2} \cmidrule(lr){3-3} \cmidrule(lr){4-4} \cmidrule(l){5-5}
    None & {23.78} / {0.713} / {0.190} & {23.95} / {0.720} / {0.183} & {23.83} / {0.715} / {0.194} & {23.99} / {0.726} / {0.187} \\
    \addlinespace
    2D   & {23.50} / {0.700} / {0.190} & {23.83} / {0.707} / {0.189} & 24.47 / 0.739 / \colorbox{yellow!100!black!30}{0.170} & {24.12} / {0.727} / {0.189} \\
    \addlinespace
    $\bm{P}$    & 24.40 / 0.738 / {0.174} & \colorbox{red!40}{24.70} / \colorbox{red!40}{0.749} / \colorbox{yellow!100!black!30}{0.170} & {24.18} / {0.727} / {0.177} & 24.29 / 0.734 / {0.176} \\
    \addlinespace
    2D+$\bm{P}$ & \colorbox{yellow!100!black!30}{24.60} / \colorbox{orange!30}{0.745} / \colorbox{red!40}{0.165} & \colorbox{orange!30}{24.64} / \colorbox{yellow!100!black!30}{0.744} / \colorbox{red!40}{0.165} & {24.07} / {0.727} / {0.182} & {24.28} / {0.733} / {0.180} \\
    \bottomrule
  \end{tabular}
\end{table}

\subsection{Why Does the Projection Matrix in Value-Output Positional Encoding Cause Performance Degradation?}
\label{sec:explor_prope_2}
We explore the cause of the performance degradation on MVImgNet2~\cite{mvimgnet2} by decomposing the projection matrix in vo-pe into its constituent elements: camera intrinsics $\bm{K}$, rotation $\bm{R}$, and translation $\bm{t}$.
Note that for the query-key pair, we maintain the ``2D + $\bm{P}$'' configuration.

\paragraph{Results}

\begin{table}[t]
  \caption{\textbf{Results for various camera parameter combinations on MVImgNet2.} We evaluate the effects of each camera parameter, the intrinsics $\bm{K}$, the rotation $\bm{R}$ and the translation $\bm{t}$ in the value-output positional encoding (vo-pe). All values represent the mean over 3 seeds. For each metric, the best, second-best, and third-best results are highlighted in red, orange, and yellow, respectively. Performance decreases whenever $\bm{R}$ and $\bm{t}$ are applied simultaneously.}
  \label{tab:value_pe_ablation}
  \centering
  \setlength{\tabcolsep}{4pt}
  \setlength{\fboxsep}{0pt}
  \begin{tabular}{@{}l cccccccc@{}}
    \toprule
    vo-pe & $\bm{KRt}$ & $\bm{Rt}$ & $\bm{KR}$ & $\bm{Kt}$ & $\bm{K}$ & $\bm{R}$ & $\bm{t}$ & None \\
    \midrule
    PSNR$\uparrow$    & $23.30$ & $23.63$ & $24.18$ & \colorbox{orange!30}{$24.22$} & \colorbox{yellow!100!black!30}{$24.19$} & $24.17$ & \colorbox{red!40}{$24.23$} & {$24.18$} \\
    SSIM$\uparrow$    & $0.711$ & $0.723$ & \colorbox{yellow!100!black!30}{$0.737$} & \colorbox{orange!30}{$0.738$} & $0.736$ & \colorbox{yellow!100!black!30}{$0.737$} & \colorbox{red!40}{$0.739$} & $0.736$ \\
    LPIPS$\downarrow$ & $0.201$ & $0.189$ & \colorbox{yellow!100!black!30}{$0.173$} & \colorbox{orange!30}{$0.172$} & $0.174$ & {$0.175$} & \colorbox{red!40}{$0.171$} & \colorbox{yellow!100!black!30}{$0.173$} \\
    \bottomrule
  \end{tabular}
\end{table}

As shown in \cref{tab:value_pe_ablation}, each metric drops in the combination of $(\bm{K}, \bm{R}, \bm{t})$, which is projection matrix itself, and $(\bm{R}, \bm{t})$ settings.
The common characteristic of these settings is that rotation $\bm{R}$ and translation $\bm{t}$ are applied simultaneously.

To formulate a hypothesis for why the simultaneous application of $\bm{R}$ and $\bm{t}$ degrades performance, we expand VO positional encoding in \cref{eq:prope_attn}.
The relative transformation between $i$-th and $j$-th tokens is described as follows by extracting only the top four dimensions of the values $\bm{v}_j\in\mathbb{R}^d$:
\begin{align}
    \bm{P}_i \bm{P}_j^{-1} \bm{v}_{j, [1:4]}
    &=
    \begin{pmatrix}
        \bm{K}_i \bm{R}_i \bm{R}_j^T \bm{K}_j^{-1} \bm{v}_{j, [1:3]} -\bm{K}_i \bm{R}_i \bm{R}_j^T \bm{t}_j \bm{v}_{j, 4} + \bm{K}_i \bm{t}_i \bm{v}_{j, 4} \\
        \bm{v}_{j, 4}
    \end{pmatrix},
    \label{eq:primal}
\end{align}
where $\bm{x}_{[i:j]}\in\mathbb{R}^{j-i+1}$ denotes elements from $i$-th to $j$-th of $\bm{x}\in\mathbb{R}^d$ and $\bm{K}_i\in\mathbb{R}^{3\times 3}$, $\bm{R}_i\in\mathbb{R}^{3\times 3}$, and $\bm{t}_i\in\mathbb{R}^{3}$ denote the camera intrinsic matrix, the rotation matrix, and the translation vector for the viewpoint corresponding to the $i$-th token, respectively.
This equation indicates that rotation $\bm{R}_i$ and translation $\bm{t}_i$ are coupled within the same top three dimensions.
For datasets such as RealEstate10K and SpatialVidHQ, where typically only one of $\boldsymbol{R}_i$ or $\boldsymbol{t}_i$ varies (\textit{e.g.}, during a shift or pan), this coupling poses no issue, and no performance degradation is observed.
However, given that performance degradation occurred solely on MVImgNet2, we consider the following hypothesis:
\textit{The coupling of rotation and translation hinders the modeling of scenes where these two components change simultaneously.}
This hypothesis suggests that although vo-pe is originally intended for coordinate-system alignment, in practice it may be implicitly required to embed camera information in an identifiable form.

As formally shown in \cref{sec:iden_vope}, the projection matrix-based positional encoding is insufficient to identify the rotation and translation contributions at the per-token level.
Since identifying camera parameters is essential for multi-view stereo tasks such as NVS, this non-identifiability forces the model to disambiguate camera parameters indirectly by aggregating contextual cues across tokens, complicating the learning process.
Therefore, we propose a positional encoding that ensures per-token identifiability of camera parameters.

%% file: sections/04_method.tex
\section{Method}
\label{sec:method}
Building on the hypothesis from \cref{sec:explor_prope_2}, we propose Decoupled Pose Positional Encoding (DPPE), a camera-based positional encoding that retains all camera information while explicitly decoupling rotation $\bm{R}$ and translation $\bm{t}$.
We introduce two variants: DPPE\textsubscript{tAdd}, which explicitly partitions the feature space into dedicated dimensions for storing $\bm{R}$ and $\bm{t}$, and DPPE\textsubscript{dual}, which uses the dual form of the projection matrix.
As in PRoPE~\cite{prope}, we also use RoPE~\cite{rope} for the 2D patch coordinates; we omit that below for brevity.

\subsection{DPPE\textsubscript{tAdd}: Decoupling via Additive Translation}\label{sec:tAdd}
DPPE\textsubscript{tAdd} explicitly decouples $\bm{R}$ and $\bm{t}$ by dividing the token dimensions into two equal halves and applying rotation- and translation-based encodings to each half, respectively.
The embeddings of rotation-based positional encoding, $\gamma_{\mathrm{DPPE}_{\mathrm{tAdd}, \bm{R}}}$ are defined using a transformation by the intrinsic matrix $\bm{K}_i\in\mathbb{R}^{3\times3}$ and the rotation matrix $\bm{R}_i\in\mathbb{R}^{3\times3}$ for the $i$-th token as follows:
\begin{align}
    \gamma^{\mathrm{DPPE}_{\mathrm{tAdd}, \bm{R}}}_{d/6}(i) = \bm{I}_{d/6} \otimes \bm{K}_i\bm{R}_i.
    \label{eq:dppe_tadd_r}
\end{align}
This embedding is then multiplied by the rotation-block components of queries, keys, values, and outputs as in \cref{eq:prope_attn}.
Importantly, the attention weights for the rotation block are computed solely from the rotation-block similarity, independently of the translation block (see \cref{sec:self-attn_dppe} for the full self-attention formulation).
The embeddings of translation-based positional encoding are defined as follows:
\begin{align}
    \gamma^{\mathrm{DPPE}_{\mathrm{tAdd}, \bm{t}}}_{d/6}(i) = \bm{1}_{d/6} \otimes \bm{t}_i,
    \label{eq:dppe_tadd_t}
\end{align}
where $\bm{t}_i\in\mathbb{R}^3$ denotes the translation vector for the $i$-th token and $\bm{1}_{d/6} \in \mathbb{R}^{d/6}$ is the all-ones column vector.
Then, $\gamma^{\mathrm{DPPE}_{\mathrm{tAdd}, \bm{t}}}$ is simply added to queries, keys, values, and outputs.

The detailed formulation of self-attention with DPPE\textsubscript{tAdd} can be found in \cref{sec:self-attn_dppe}.

\subsection{DPPE\textsubscript{dual}: Decoupling via the Dual Projection Matrix}\label{sec:dual}
We also propose DPPE\textsubscript{dual}, which applies the dual form of the projection matrix $\bm{P}^{-T}_i \in \mathbb{R}^{4 \times 4}$~\cite{hartley2003multiple} in place of $\bm{P}_i$ in ~\cref{eq:prope}: $\gamma^{\mathrm{DPPE_\mathrm{dual}}}(i) = \bm{I}_{d/4} \otimes \bm{P}^{-T}_i$.
To examine its properties, we extract the top four dimensions of the value-output transformation, analogously to \cref{eq:primal}:
\begin{align}
    \bm{P}_i^{-T} \bm{P}_j^{T} \bm{v}_{j,[1:4]}
    & =
    \begin{pmatrix}
        \bm{K}_i^{-T} \bm{R}_i \bm{R}_j^T \bm{K}_j^T \bm{v}_{j,[1:3]}\\
        -\bm{t}_i^T \bm{R}_i \bm{R}_j^T \bm{K}_j^T \bm{v}_{j,[1:3]} + \bm{t}_j^T \bm{K}_j^T \bm{v}_{j,[1:3]} + \bm{v}_{j,4}
    \end{pmatrix}
    \label{eq:dual_v}
\end{align}

Although translation $\bm{t}$ is not entirely decoupled from rotation $\bm{R}$, the first three dimensions of \cref{eq:dual_v} contain no $\bm{t}$ terms, so the influence of $\bm{t}$ is isolated within the fourth dimension and can be uniquely identified there.
We provide a formal proof of this property in \cref{sec:iden_vope}.
Consequently, this design is expected to yield practically the same effect as being fully decoupled, and we verify its effectiveness in the next section.

%% file: sections/05_experiments.tex
\section{Experiments}
\label{sec:exp}
In this section, we verify the effectiveness of DPPE in the NVS task, especially for the scenes where camera rotation and translation change simultaneously.
As we discussed in \cref{sec:related_work}, we compare DPPE with PRoPE~\cite{prope} as our baseline method, and we also compare DPPE with GTA~\cite{prope} in \cref{sec:degration}.
We focus the main evaluation on NVS because it requires precise underlying 3D reconstruction and cannot be solved without accurate camera information, making it particularly well-suited for validating camera-based positional encodings. 
We also report the results on the multi-view depth estimation tasks in \cref{sec:depth_exp}, where DPPE likewise demonstrates equivalent to or better than the baseline.

\paragraph{Training Setup}
For the model, we adopt the large-scale training configuration used in PRoPE~\cite{prope}, a 12-layer Transformer with a hidden dimension of 768 and an MLP dimension of 3072, and train the model for 320k iterations with a variable number of multi-view inputs ranging from 2 to 4. All experiments are conducted using three different random seeds.
Detailed experimental settings are provided in \cref{app:exp}.

\subsection{Evaluation of DPPE as the Value-Output Positional Encoding}
\label{sec:val_exp}

In this experiment, we applied DPPE\textsubscript{tAdd} and DPPE\textsubscript{dual} to the value-output pair and evaluate them on MVImgNet2~\cite{mvimgnet2}.
Note that the positional encoding for the query-key pair is the same as in PRoPE.

\begin{table}[t]
  \centering
  \caption{\textbf{Effect of applying DPPE to different attention components on MVImgNet2.} (a)DPPE applied only to the value-output pair. (b) DPPE applied to both the query-key and value-output pairs. All values represent the mean over 3 seeds.}
  \label{tab:pe_ablation}
  \vspace*{-8pt}
  \setlength{\tabcolsep}{4pt} 
  \setlength{\fboxsep}{0pt}
  \begin{subtable}[t]{0.48\textwidth}
    \captionsetup{aboveskip=0pt}
    \centering
    \caption{Value-output only.}
    \label{tab:k_pe_ablation}
    \begin{tabular}{@{}lccc@{}}
      \toprule
      Method & PSNR$\uparrow$ & SSIM$\uparrow$ & LPIPS$\downarrow$ \\
      \midrule
      PRoPE        & {$22.91$} & {$0.696$} & {$0.217$} \\
      DPPE\textsubscript{tAdd}        & $\mathbf{24.02}$ & $\mathbf{0.730}$ & {$0.180$} \\
      DPPE\textsubscript{dual}            & {$23.91$} & {$0.725$} & {$\mathbf{0.178}$} \\
      \bottomrule
    \end{tabular}
  \end{subtable}
  \hfill
  \begin{subtable}[t]{0.48\textwidth}
    \captionsetup{aboveskip=0pt}
    \centering
    \caption{Both query-key and value-output.}
    \label{tab:kv_pe_ablation}
    \begin{tabular}{@{}lccc@{}}
      \toprule
      Method & PSNR$\uparrow$ & SSIM$\uparrow$ & LPIPS$\downarrow$ \\
      \midrule
      PRoPE           & {$22.91$} & {$0.696$} & {$0.217$} \\
      DPPE\textsubscript{tAdd}        & {$23.32$} & {$0.703$} & {$0.201$} \\
      DPPE\textsubscript{dual}            & {$\mathbf{23.92}$} & {$\mathbf{0.726}$} & {$\mathbf{0.178}$} \\
      \bottomrule
    \end{tabular}
  \end{subtable}
\end{table}

\begin{table}[t]
    \centering
    \setlength{\tabcolsep}{4pt}
    \caption{\textbf{Comparison of camera-based positional  encodings across datasets.} We evaluate DPPE and PRoPE on MVImgNet2, RealEstate10K, and SpatialVidHQ. All values represent the mean over 3 seeds. The best scores are shown in \textbf{bold}.}
    \label{tab:main_train}
    \begin{tabular}{l ccc ccc ccc}
        \toprule
        \multirow{2}{*}{Method} & \multicolumn{3}{c}{MVImgNet2} & \multicolumn{3}{c}{RealEstate10K} & \multicolumn{3}{c}{SpatialVidHQ} \\
        \cmidrule(lr){2-4} \cmidrule(lr){5-7} \cmidrule(lr){8-10}
        & PSNR$\uparrow$ & SSIM$\uparrow$ & LPIPS$\downarrow$ & PSNR$\uparrow$ & SSIM$\uparrow$ & LPIPS$\downarrow$ & PSNR$\uparrow$ & SSIM$\uparrow$ & LPIPS$\downarrow$ \\
        \cmidrule(lr){1-1} \cmidrule(lr){2-4} \cmidrule(lr){5-7} \cmidrule(lr){8-10}
        PRoPE~\cite{prope} & $22.91$ & $0.696$ & $0.217$ & $24.14$ & $0.851$ & $0.137$ & $19.21$ & $0.626$ & $0.268$ \\
        DPPE\textsubscript{dual} & $23.92$ & $0.726$ & $\mathbf{0.178}$ & $24.16$ & $0.850$ & $0.138$ & $19.26$ & $0.629$ & $0.266$ \\
        DPPE\textsubscript{tAdd} & $\mathbf{24.02}$ & $\mathbf{0.730}$ & $0.180$ & $\mathbf{24.27}$ & $\mathbf{0.855}$ & $\mathbf{0.133}$ & $\mathbf{19.28}$ & $\mathbf{0.630}$ & $\mathbf{0.265}$ \\
        \bottomrule
    \end{tabular}
\end{table}

\paragraph{Results}
As shown in \cref{tab:k_pe_ablation}, both DPPE variants outperform PRoPE in PSNR, SSIM, and LPIPS, demonstrating that decoupling rotation and translation in the value-output positional encoding leads to clear performance gains.
The fact that DPPE\textsubscript{dual} also improves performance---despite not fully separating the rotation- and translation-storing dimensions---indicates that what matters is making camera parameters identifiable, rather than fully disjoint storage.

We also show the qualitative results in \cref{fig:viz}.
Compared to PRoPE, DPPE successfully synthesizes more detailed textures.
When the target viewpoint's camera information is ambiguous, recovering coarse structure may still be possible, but reproducing fine textures becomes highly challenging.
This observation suggests that DPPE facilitates the model in identifying camera parameters more easily.
A comparison between the two variants further shows that DPPE\textsubscript{tAdd} reproduces slightly finer details than DPPE\textsubscript{dual}, which we attribute to its more direct decoupling of rotation and translation.

\subsection{Evaluation of DPPE as the Query-Key Positional Encoding}
We evaluate DPPE by applying each variant to both query-key and value-output pairs on MVImgNet2.
This allows us to verify whether DPPE works not only as an effect on the values and outputs, but also as a similarity bias in the computation of the attention matrix. 

\paragraph{Results}
As shown in \cref{tab:kv_pe_ablation}, DPPE\textsubscript{tAdd} still outperforms PRoPE, but its performance is lower than when it is applied solely to the value-output pair.
DPPE\textsubscript{dual}, in contrast, maintains its performance under this setting.
This asymmetric behavior suggests that the positional encodings for query-key and value-output pairs serve distinct roles.

Unlike the value-output positional encoding, the positional encodings for queries and keys are expected to bias attention by token similarity; however, additive translation injection in DPPE\textsubscript{tAdd} introduces a term $\boldsymbol{t}_i^{\top}\boldsymbol{t}_j$ in the query-key dot product that depends only on absolute camera positions and grows with the scene scale.
This scene-scale-dependent bias overwhelms the genuine similarity signal in scenes with large translations, undermining the intended role of the query-key positional encoding.
Still, it is noteworthy that DPPE\textsubscript{tAdd} remains superior to PRoPE.
In contrast, since DPPE\textsubscript{dual} uses the dual form of the projection matrix, it is expected to behave similarly to PRoPE, which uses the primal form.
Thus, DPPE\textsubscript{dual} keeps its performance even when it is also applied to queries and keys.
Based on these findings, in the subsequent experiments we configure DPPE\textsubscript{tAdd} to use PRoPE for the query-key positional encodings (and DPPE\textsubscript{tAdd} for the value-output positional encoding), while DPPE\textsubscript{dual} is used for both the query-key and value-output positional encodings.

\begin{figure}[t]
    \centering
    \includegraphics[width=0.95\textwidth]{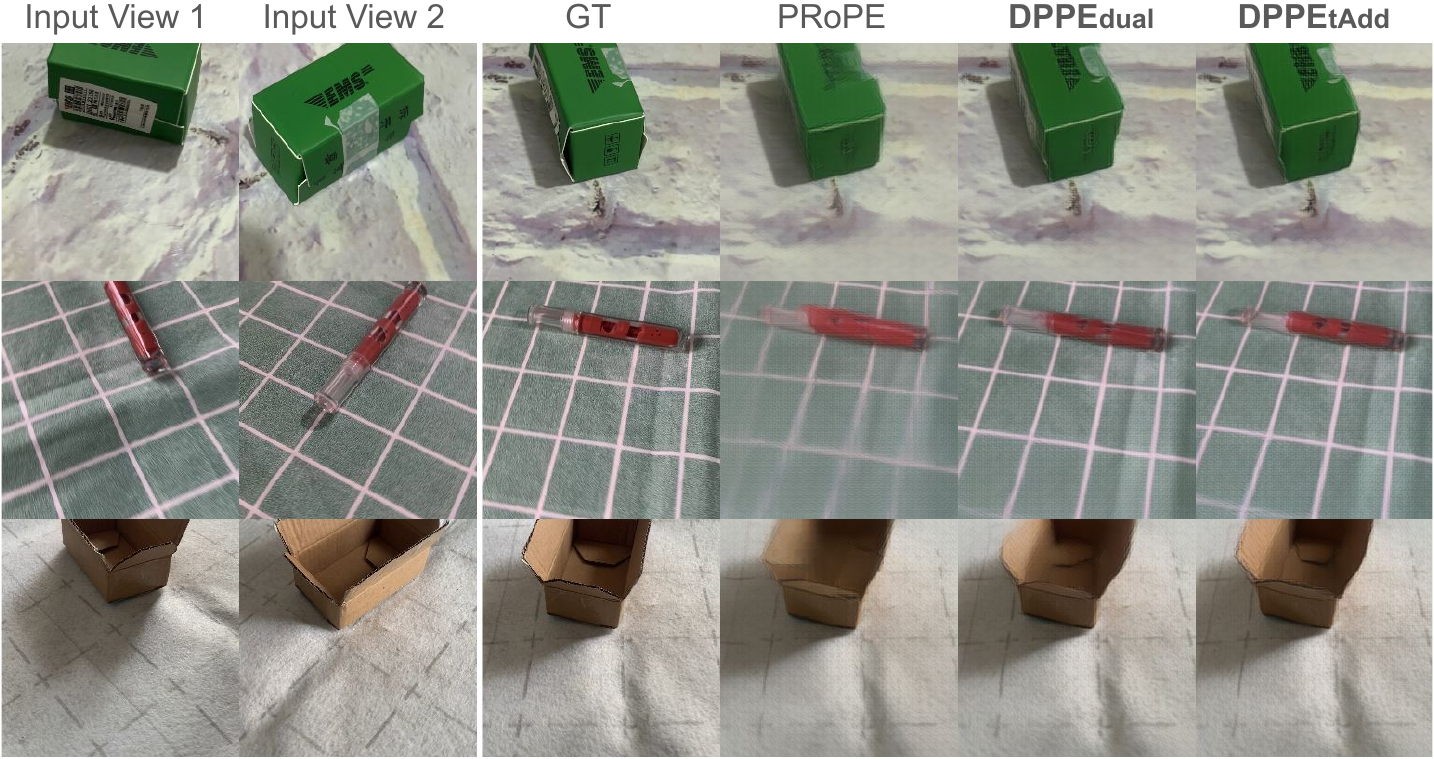
    }
    \caption{\textbf{Qualitative comparison of DPPE and PRoPE on MVImgNet2.} DPPE recovers more detailed textures compared to PRoPE, indicating an improved ability to disambiguate camera parameters. DPPE\textsubscript{tAdd} produces slightly finer details than DPPE\textsubscript{dual}, reflecting its more explicit decoupling of rotation and translation.}
    \label{fig:viz}
\end{figure}

\begin{figure}[t]
    \centering
    \vspace*{-8pt}
    \begin{subfigure}[t]{0.95\textwidth}
        \includegraphics[width=\textwidth]{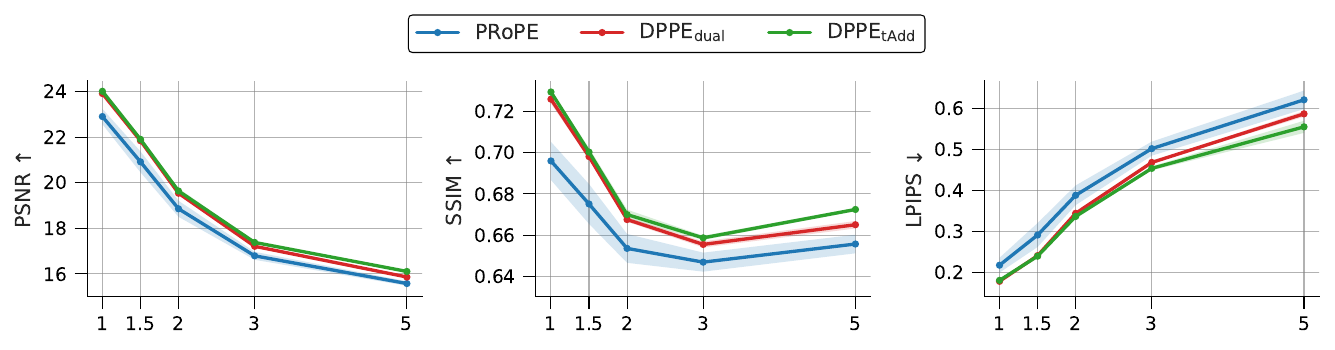}
        \caption{Evaluation results with zooming-in $\times$ [1.0, 5.0].}
        \label{fig:zoom_ood_plot}
    \end{subfigure}
    \begin{subfigure}[t]{0.95\textwidth}
        \includegraphics[width=\textwidth]{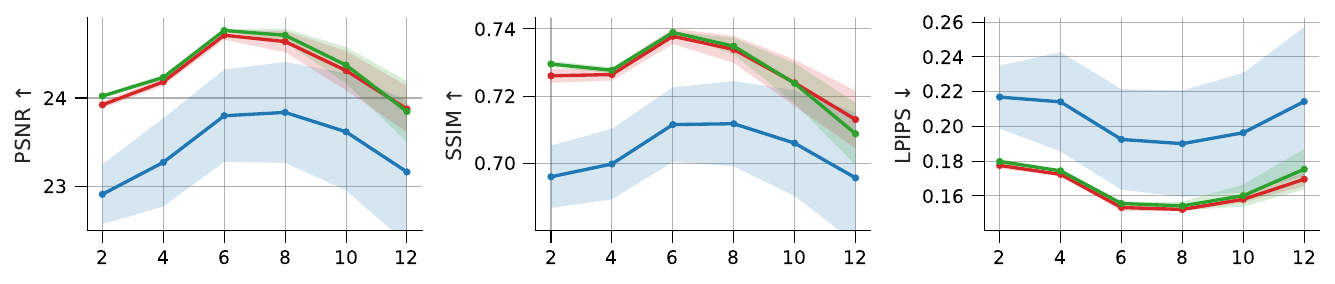}
        \caption{Evaluation results with a variable number of multi-view inputs from 2 to 12.}
        \label{fig:input_view_ood_plot}
    \end{subfigure}
    \caption{\textbf{Evaluation under extrapolation settings.} We evaluate extrapolation capability with respect to the number of viewpoints and the zoom-in factor on MVImgNet2. DPPE constantly outperforms PRoPE in both settings. All values represent the mean over 3 seeds, and the shaded regions indicate the standard deviation.}
    \label{fig:ood_plot}
\end{figure}

\subsection{Generalization and Extrapolation}
In this experiment, we compare DPPE with PRoPE \cite{prope} on MVImgNet2, RealEstate10K, and SpatialVidHQ to evaluate generalization performance across datasets.
We also evaluate the extrapolation performance---that is, the performance under conditions deviating from the training setup---regarding the number of multi-view inputs and increased focal lengths (i.e., zoom-in) on MVImgNet2.

\paragraph{Results}
As shown in \cref{tab:main_train}, DPPE matches or outperforms PRoPE across all three datasets.
On MVImgNet2 in particular, PSNR improves by more than 1.2 dB, confirming effectiveness on data where rotation and translation vary significantly together. 
The marginal gap on SpatialVidHQ is likely due to its predominantly simple camera trajectories, where the benefit of decoupling rotation and translation is inherently limited.

We also show the results for the zoom-in setting in \cref{fig:zoom_ood_plot} and the results for the various number of viewpoints setting in \cref{fig:input_view_ood_plot}.
In both settings, DPPE consistently outperforms PRoPE, indicating strong generalization to unseen test conditions.

%% file: sections/07_conclusion.tex
\section{Limitations}
\label{sec:limitations}
As shown in \cref{tab:main_train}, the proposed method does not yield significant performance improvements over PRoPE on SpatialVidHQ.
We attribute this to the predominantly simple camera trajectories in the dataset, where rotation and translation rarely change simultaneously and thus the rotation–translation coupling in PRoPE's value-output positional encoding does not manifest as a critical bottleneck.
We acknowledge that such simple-trajectory regimes (\textit{e.g.}, pure panning, dolly shots, or near-planar motions) cover a non-trivial portion of practical applications, including indoor walkthroughs, drone footage with stabilized gimbals, and many video generation scenarios.
In these settings, the benefit of decoupling rotation and translation is inherently limited, and DPPE offers no clear advantage over existing methods.
Nevertheless, as the field moves toward larger-scale and more diverse training corpora, scenes with complex coupled motion will become increasingly prevalent, and the limitations of PRoPE identified in this work are likely to emerge as a critical bottleneck.

\section{Conclusion}
\label{sec:conclusion}

In this study, we proposed Decoupled Pose Positional Encoding (DPPE), a camera-based positional encoding method with two variants: DPPE\textsubscript{tAdd} and DPPE\textsubscript{dual}.
Our method is motivated by the finding that storing rotation and translation in the same dimensions causes indeterminacy in their independent identification, hindering training scalability.
DPPE\textsubscript{tAdd} explicitly assigns camera rotation and translation to separate dimensions, while DPPE\textsubscript{dual} uses the dual formulation of the projection matrix.
DPPE\textsubscript{tAdd} achieves slightly better performance than DPPE\textsubscript{dual}, but DPPE\textsubscript{dual} exhibits favorable camera-geometric properties and works effectively as a positional encoding for both the value-output and the query-key pairs.
DPPE enables stable scaled-up training on NVS and also achieves better view synthesis in extrapolation scenarios involving zoom-in and varying numbers of viewpoints, compared to PRoPE.

%% file: sections/08_appendix.tex
\newtheorem{proposition}{Proposition}
\newtheorem{proof}{Proof}

\section{Formulation of Self-Attention with DPPE\textsubscript{tAdd}}\label{sec:self-attn_dppe}

We provide the explicit formulation of the self-attention with DPPE\textsubscript{tAdd}. For clarity, we omit the RoPE component for 2D patch coordinates and focus solely on the rotation- and translation-based positional encodings.

\subsection{Notation and Feature Decomposition}

DPPE\textsubscript{tAdd} partitions the $d$-dimensional token features into two halves of equal dimension $d/2$, dedicated to rotation and translation respectively (we assume $d$ is divisible by $6$). For the $i$-th token, we write
\begin{align}
    \bm{q}_i = \begin{bmatrix} \bm{q}_i^R \\ \bm{q}_i^t \end{bmatrix}, \quad
    \bm{k}_i = \begin{bmatrix} \bm{k}_i^R \\ \bm{k}_i^t \end{bmatrix}, \quad
    \bm{v}_i = \begin{bmatrix} \bm{v}_i^R \\ \bm{v}_i^t \end{bmatrix} \in \mathbb{R}^d,
\end{align}
where the superscripts $R$ and $t$ denote the first $d/2$ dimensions (rotation block) and the last $d/2$ dimensions (translation block), respectively. 
The positional encodings are
\begin{align}
    \gamma^R(i) &:= \gamma^{\mathrm{DPPE}_{\mathrm{tAdd},R}}(i) = I_{d/6} \otimes (K_i R_i) \in \mathbb{R}^{(d/2) \times (d/2)}, \\
    \gamma^t(i) &:= \gamma^{\mathrm{DPPE}_{\mathrm{tAdd},t}}(i) = \bm{1}_{d/6} \otimes \bm{t}_i \in \mathbb{R}^{d/2},
\end{align}
where $\bm{1}_{d/6} \in \mathbb{R}^{d/6}$ is the all-ones column vector.

\subsection{Block-wise Attention Computation}

Since $\gamma^R$ and $\gamma^t$ act on disjoint dimension blocks, DPPE\textsubscript{tAdd} computes the attention score as a \emph{concatenation} of two block-wise scores rather than as their sum:
\begin{align}
    \bm{s}_{ij} \;=\; \begin{bmatrix} s_{ij}^R \\ s_{ij}^t \end{bmatrix} \in \mathbb{R}^{2},
\end{align}
\begin{align}
    s_{ij}^R &\;=\; (\bm{q}_i^R)^{\top} \gamma^R(i) \, \gamma^R(j)^{-1} \bm{k}_j^R \,\big/\, \sqrt{d/2},\\
    s_{ij}^t &\;=\; \bigl( \bm{q}_i^t + \gamma^t(i) \bigr)^{\top} \bigl( \bm{k}_j^t + \gamma^t(j) \bigr)\,\big/\, \sqrt{d/2}.
\end{align}

\paragraph{Block-wise softmax.}
The attention weights for the two blocks are computed by independent softmax normalization over the source token index, so that each block has its own attention pattern:
\begin{align}
    \alpha^R_{ij} &\;=\; \frac{\exp\!\left(s_{ij}^R\right)}
    {\sum_{j'} \exp\!\left(s_{ij'}^R\right)}, &
    \alpha^t_{ij} &\;=\; \frac{\exp\!\left(s_{ij}^t\right)}
    {\sum_{j'} \exp\!\left(s_{ij'}^t\right)}.
\end{align}

\paragraph{Block-wise output.}
The output $\bm{o}'_i = [\,(\bm{o}'_i)^R \,;\, (\bm{o}'_i)^t\,]$ is computed block-wise using the corresponding attention weights:
\begin{align}
    (\bm{o}'_i)^R &\;=\; \gamma^R(i) \sum_{j} \alpha^R_{ij} \, \gamma^R(j)^{-1} 
    \bm{v}_j^R, \label{eq:dppetadd_oR}\\
    (\bm{o}'_i)^t &\;=\; \gamma^t(i) + \sum_{j} \alpha^t_{ij} 
    \bigl( \bm{v}_j^t - \gamma^t(j) \bigr). \label{eq:dppetadd_ot}
\end{align}

This block-wise decoupling ensures that the rotation pattern $\alpha^R$ depends only on rotation-related similarities (via $\gamma^R$) and the translation pattern $\alpha^t$ depends only on translation-related similarities (via $\gamma^t$); the two computations never share information through the attention weights.

\subsection{Per-Token Value Transformation}

The per-token value-side contributions used in \cref{sec:iden_vope} are recovered as follows. For the rotation block,
\begin{align}
    \gamma^R(i) \, \gamma^R(j)^{-1} \bm{v}_j^R \;=\; \left( I_{d/6} \otimes K_i R_i R_j^{\top} K_j^{-1} \right) \bm{v}_j^R,
\end{align}
which depends on $(R_i, R_j, K_i, K_j)$ but not on $(\bm{t}_i, \bm{t}_j)$.
For the translation block, since $\sum_j \alpha^t_{ij} = 1$,
\begin{align}
    (\bm{o}'_i)^t \;=\; \sum_{j} \alpha^t_{ij} \bigl( \bm{v}_j^t + \bm{1}_{d/6} \otimes (\bm{t}_i - \bm{t}_j) \bigr),
\end{align}
so that the per-token contribution is $\bm{v}_j^t + \bm{1}_{d/6} \otimes (\bm{t}_i - \bm{t}_j)$, depending on $(\bm{t}_i, \bm{t}_j)$ but not on $(R_i, R_j)$.
The disjoint dependence is the basis of the injectivity proof in Prop.~\ref{prop:dppe_tadd_inj}.

\section{Identifiability Analysis of Camera Parameters in the Value-Output Positional Encoding}\label{sec:iden_vope}

This appendix formally shows that PRoPE's value-output positional encoding is non-injective in $(\bm{R}, \bm{t})$ at the per-token level, while DPPE restores identifiability.
We use the notation of \cref{sec:explor_prope}: $\bm{P}_i = \bm{K}_i[\bm{R}_i \mid \bm{t}_i]$ is the projection matrix lifted to $\mathbb{R}^{4\times 4}$, and $\bm{v}_{j,[1:3]}$, $\bm{v}_{j,4}$ denote the first three and fourth entries of $\bm{v}_j$.
The analysis concerns the top four dimensions; the remaining blocks (cf. \cref{eq:prope}) follow by identical construction.

\subsection{Non-Identifiability of PRoPE's Value-Output Positional Encoding}

\begin{proposition}[Non-identifiability of PRoPE's value-output positional encoding]
\label{prop:prope_noninj}
Let $\bm{y}$ denote the top three dimensions of PRoPE's value-output positional encoding output (\cref{eq:primal}). For any $\bm{v}_j$ with
$\bm{v}_{j,4} \neq 0$ and any tangent perturbation
$\delta\bm{R}_i$, the translation perturbation
\begin{equation}
\delta\bm{t}_i = -\delta\bm{R}_i \bm{R}_j^{\top}\!\left(\frac{\bm{K}_j^{-1}\bm{v}_{j,[1:3]}}{\bm{v}_{j,4}} - \bm{t}_j\right)
\label{eq:app:prope_compensation}
\end{equation}
satisfies $\delta\bm{y} = \bm{0}$.
\end{proposition}

\begin{proof}
Differentiating \cref{eq:primal} with respect to $(\bm{R}_i, \bm{t}_i)$ gives $\delta\bm{y} = \bm{K}_i \delta\bm{R}_i \bm{R}_j^{\top}\bigl(\bm{K}_j^{-1}\bm{v}_{j,[1:3]} - \bm{t}_j \bm{v}_{j,4}\bigr) + \bm{K}_i \delta\bm{t}_i \bm{v}_{j,4}$.
Since $\bm{K}_i$ is invertible and $\bm{v}_{j,4} \neq 0$, setting $\delta\bm{y} = \bm{0}$ yields \cref{eq:app:prope_compensation}.
\end{proof}

The compensating $\delta\bm{t}_i$ depends on $\bm{v}_j$ through both $\bm{v}_{j,[1:3]}$ and $\bm{v}_{j,4}$, so no fixed downstream transformation can recover $(\bm{R}_i, \bm{t}_i)$ from a single token; disambiguation requires aggregating across tokens with diverse value content.

\subsection{Identifiability of DPPE}

\begin{proposition}[DPPE\textsubscript{tAdd} is injective]
\label{prop:dppe_tadd_inj}
In DPPE\textsubscript{tAdd}, rotation acts on the upper half via $\bm{K}\bm{R}$ (\cref{eq:dppe_tadd_r}) and translation acts additively on the lower half (\cref{eq:dppe_tadd_t}).
Hence rotation and translation perturbations affect disjoint output coordinates and cannot cancel; the map $(\bm{R}, \bm{t}) \mapsto T(\bm{R}, \bm{t}; \bm{v}_j)$ is injective for any $\bm{v}_j$ with $\bm{v}_{j,[1:d/4]} \neq \bm{0}$.
\end{proposition}

\begin{proof}
From \cref{eq:dppe_tadd_r}, the rotation block of the value-domain output (the upper half) is
\begin{equation}
T_R = \bigl(\bm{I}_{d/6} \otimes \bm{K}_i \bm{R}_i \bm{R}_j^{\top} \bm{K}_j^{-1}\bigr) \bm{v}_j^R,
\end{equation}
which depends on $(\bm{R}_i, \bm{R}_j)$ but not on $(\bm{t}_i, \bm{t}_j)$. From \cref{eq:dppe_tadd_t}, the translation block (the lower half) is
\begin{equation}
T_t = \bm{v}_j^t + \bm{I}_{d/6} \otimes (\bm{t}_i - \bm{t}_j),
\end{equation}
which depends on $(\bm{t}_i, \bm{t}_j)$ but not on $(\bm{R}_i, \bm{R}_j)$. Differentiating yields
\begin{align}
\delta T_R &= \bm{I}_{d/6} \otimes \bm{K}_i\!\left(\delta\bm{R}_i \bm{R}_j^{\top} + \bm{R}_i \delta\bm{R}_j^{\top}\right)\!\bm{K}_j^{-1} \bm{v}_j^R, \\
\delta T_t &= \bm{I}_{d/6} \otimes (\delta\bm{t}_i - \delta\bm{t}_j).
\end{align}
Since $T_R$ and $T_t$ occupy disjoint coordinate blocks, their perturbations cannot cancel one another.
In particular, any non-zero rotation perturbation produces a non-zero $\delta T_R$ (provided $\bm{v}_j^R \neq \bm{0}$ and $\bm{K}_i, \bm{K}_j$ invertible), which no translation perturbation can compensate; symmetrically for translations.
\end{proof}

\begin{proposition}[DPPE\textsubscript{dual} is sequentially identifiable]
\label{prop:dppe_dual_partial}
The first three of the top four dimensions of DPPE\textsubscript{dual}'s output (\cref{eq:dual_v}) depend only on rotations and intrinsics, with no translation contribution. Hence rotations are uniquely identifiable from these dimensions; conditional on rotations, the fourth dimension is an affine function of $(\bm{t}_i, \bm{t}_j)$ with known coefficients, making translations identifiable.
\end{proposition}

\begin{proof}
From \cref{eq:dual_v}, the top three output dimensions are
\begin{equation}
\bm{y}_R := \bm{K}_i^{-\top} \bm{R}_i \bm{R}_j^{\top} \bm{K}_j^{\top} \bm{v}_{j,[1:3]},
\label{eq:app:dual_yR}
\end{equation}
which depends on $(\bm{R}_i, \bm{R}_j)$ but contains no $(\bm{t}_i, \bm{t}_j)$ terms. Differentiating with respect to translations yields $\delta\bm{y}_R = \bm{0}$, so no translation perturbation can affect $\bm{y}_R$, and rotations are identifiable from $\bm{y}_R$ independently of translations.

Conditional on $(\bm{R}_i, \bm{R}_j)$ being fixed, the fourth output dimension becomes an affine function of $(\bm{t}_i, \bm{t}_j)$:
\begin{equation}
y_t = -\bm{t}_i^{\top} \bm{a} + \bm{t}_j^{\top} \bm{b} + \bm{v}_{j,4},
\label{eq:app:dual_yt}
\end{equation}
where $\bm{a} := \bm{R}_i \bm{R}_j^{\top} \bm{K}_j^{\top} \bm{v}_{j,[1:3]}$ and $\bm{b} := \bm{K}_j^{\top} \bm{v}_{j,[1:3]}$ are known constants once rotations and intrinsics are fixed.
Provided $\bm{v}_{j,[1:3]} \neq \bm{0}$ and the intrinsics are invertible, both $\bm{a}$ and $\bm{b}$ are non-zero, so $y_t$ varies non-trivially with $(\bm{t}_i, \bm{t}_j)$ and the translation contribution is identifiable from $y_t$.
\end{proof}

\paragraph{Connection to training dynamics.}
The propositions above are per-token statements: for any single edge $(i, j)$, PRoPE's value-output transformation cannot uniquely identify $(\bm{R}_i, \bm{t}_i)$, whereas DPPE can.
The full PRoPE model can in principle disambiguate $(\bm{R}_i, \bm{t}_i)$ by aggregating across edges via the weighted sum $\sum_j \alpha_{ij}\, \gamma^{\mathrm{Proj}}(j)^{-1} \bm{v}_j$ in \cref{eq:prope_attn}.
We hypothesize, however, that the absence of per-token identifiability shifts the burden of disambiguation onto the rest of the network and manifests as the late-stage training degradation of PRoPE in \cref{fig:degradation}.
DPPE removes this burden by providing per-token identifiability of $(\bm{R}_i, \bm{t}_i)$. Furthermore, DPPE\textsubscript{tAdd} enforces a stronger form of decoupling at the architectural level: the rotation and translation blocks maintain independent attention patterns $\alpha^R_{ij}$ and $\alpha^t_{ij}$ (\cref{sec:self-attn_dppe}), so that translation-induced biases in attention (\textit{e.g.}, spurious bonds between distant tokens do not contaminate the aggregation of the rotation block).

\section{Experiment Details}\label{sec:exp_details}
For all tables in \cref{sec:explor_prope,sec:exp}, we present the values smoothed by applying an exponential moving average (EMA) with a decay factor of 0.75 to the evaluation results from checkpoints saved every 5k iterations between 275k and 320k iterations in order to mitigate evaluation fluctuations across iterations and ensure the reliability of the results.
Furthermore, \cref{fig:pe_comparison_b} shows the smoothed values using the same EMA with a decay factor of 0.75 across all evaluation steps from the start of training.

EMA mitigates iteration-wise variance while accurately reflecting the model's performance up to the final iteration.
Therefore, in this study, we determined and applied an decay factor for each experiment to smooth fluctuations while properly capturing the true final performance.

\subsection{Training Details of Exploring Camera-based Positional Encoding in \cref{sec:explor_prope}}
\label{app:exp_prope_detail}
For the exploratory experiments on camera-based positional encodings, we generally followed the default settings provided in the official PRoPE repository~\cite{prope}.
The image resolution is fixed at $256 \times 256$, and the world batch size is set to 64.
To facilitate analysis and reduce computational costs in these experiments, we retain the default number of input views 2 and target views 1.

We adopt a warmup-stable scheduler, which maintains a learning rate of $5 \times 10^{-4}$ after a 500-step warmup period.
This approach omits the decay phase from the warmup-stable-decay scheduler~\cite{wsd}, which has recently become the standard for the long-term training of large language models~\cite{minicpm,deepseekr1,deepseekv3,qwen3,kimi,smollm3,llada2}.
Training was conducted via distributed learning on eight H100 GPUs, taking approximately 20 hours per model.

\subsection{Training Details in \cref{sec:exp}}\label{app:exp}
The training settings for these experiments are fundamentally identical to those described in \cref{sec:explor_prope}.
However, to evaluate more practical capabilities, we trained the models by randomly varying the number of input views between 2 and 4, while keeping the number of target views fixed at 1.
As a result of this modification, the training time increased to approximately 24 hours per model.

\subsection{Multi-View Sampling Strategy}\label{sec:sampleing_strategy}
For the novel view synthesis (NVS) task, it is essential that a certain degree of subject overlap exists among the sampled images.

Regarding RealEstate10K, we used the configurations exactly as provided, since the sampling index ranges for training, along with the IDs and image indices for the evaluation data, are available in the official PRoPE~\cite{prope} repository.

For MVImgNet2 and SpatialVidHQ, we design an image sampling method because they do not provide the sampling settings: An initial image is randomly sampled, and the selection probability is configured to increase as the angular and spatial distances from the reference camera position increase.
Regarding the angle, we introduce an additional constraint: only images located within 60 degrees of any previously sampled camera are considered as candidates.

Consequently, for MVImgNet2, this approach enables the extraction of viewpoints with appropriate subject overlap while maintaining an angular difference of 60 degrees or less.
For SpatialVidHQ, this ensures that spatially distant images are appropriately sampled for camera shift, while images maintaining subject overlap are suitably selected for camera panning.

During training, this sampling process is iterated to acquire multiple images.
A randomly selected non-edge camera among the sampled images is assigned as the target view, while the remaining images are used as input views.
For evaluation, we prepare a split of approximately 6,000 instances in advance.
By applying this sampling algorithm beforehand to fix the indices, we ensured consistent evaluations on the identical dataset across every evaluation iteration.

\subsection{Descriptive Camera-Motion Statistics}
\label{sec:camera_motion_stats}

To partially assess the motion-based hypothesis in
\cref{sec:explor_prope_2}, we analyzed the training-set
camera trajectories. Since the camera poses are estimated
without metric scale, translation magnitudes are not directly
comparable across datasets. We therefore measured the angular
difference between camera viewing directions of consecutive
frames, pooling all consecutive-frame pairs across videos.

\begin{table}[t]
    \centering
    \caption{
        Angular changes in camera viewing direction between
        consecutive frames in the training sets (degrees).
    }
    \label{tab:camera_motion_stats}
    \begin{tabular}{lccc}
        \toprule
        Dataset & Mean & Median & 95th percentile \\
        \midrule
        MVImgNet2     & 11.27 & 12.58 & 17.15 \\
        RealEstate10K &  0.20 &  0.14 &  0.59 \\
        SpatialVidHQ  &  0.72 &  0.40 &  2.50 \\
        \bottomrule
    \end{tabular}
\end{table}

As shown in \cref{tab:camera_motion_stats}, MVImgNet2
exhibits substantially larger frame-to-frame changes in
viewing direction than RealEstate10K and SpatialVidHQ.
While these statistics reveal differences in camera motion,
they do not directly characterize joint rotation--translation
variation or establish a causal explanation for the
performance differences.

\section{Additional Comparison in Scaled-up Training}\label{sec:degration}

\begin{figure}[t]
  \center
  \includegraphics[width=\textwidth]{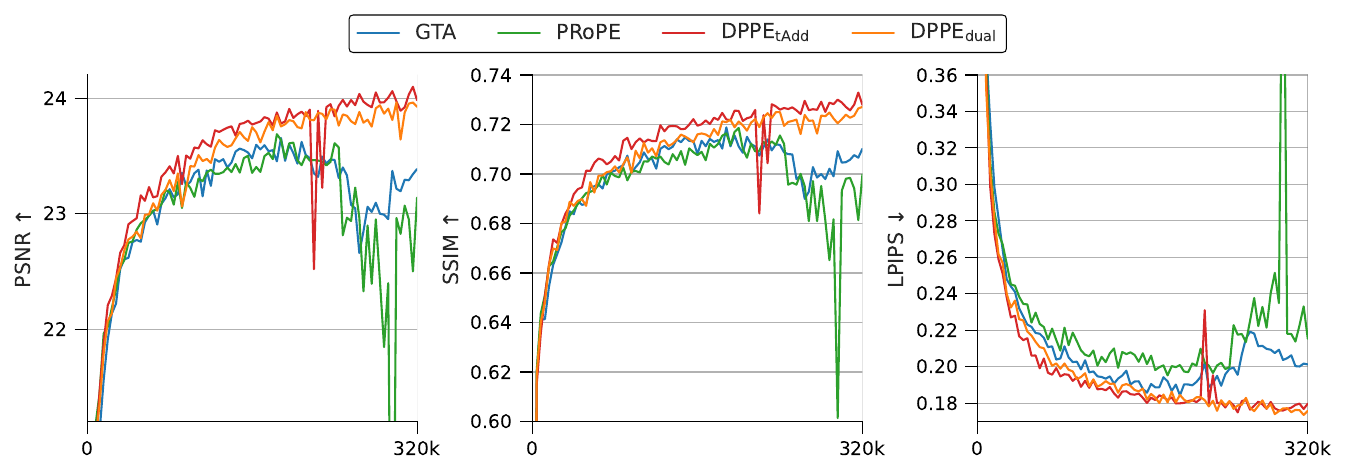}
  \caption{\textbf{Evaluation results of novel view synthesis on the MVImgNet2 dataset at each step.} While the performance of GTA~\cite{gta} and PRoPE~\cite{prope} degrades, DPPE consistently improves the performance.}
  \label{fig:degradation}
\end{figure}

When performance begins to degrade during training, selecting the best-performing checkpoint (\textit{i.e.}, early stopping) is a common strategy to maintain high performance. \cref{fig:degradation}, plotting the raw metrics on MVImgNet2 from the experiments in \cref{tab:main_train} and an additionally conducted experiment for GTA~\cite{gta}, demonstrates that the performance of PRoPE begins to drop around 240k iterations, with PSNR and LPIPS peaking at 184k iterations, and SSIM peaking at 200k iterations. Accordingly, we evaluate PRoPE using checkpoints at 184k, 200k, and 240k iterations. Similarly, GTA—which is PRoPE without intrinsic parameter information and with coupled rotation and translation—exhibits degradation similar to PRoPE in the later stages of training, as also indicated in the ($\bm{R}, \bm{t}$) row of \cref{tab:value_pe_ablation}. For GTA, PSNR and SSIM peak at 188k iterations, and LPIPS peaks at 196k iterations. The evaluation results using these optimal checkpoints are presented in \cref{tab:ckpts}.

As shown in \cref{tab:ckpts}, even when using their respective best checkpoints, both PRoPE and GTA fall short of the performance achieved by DPPE at the final iteration. Furthermore, determining the training schedule based solely on the behavior of the positional encoding—which is merely one component of the entire architecture—prevents the overall model from realizing its full potential and is therefore impractical.

\begin{table}[t]
    \captionsetup{aboveskip=0pt}
    \centering
    \caption{Comparison using the early stopping.}
    \label{tab:ckpts}
    \begin{tabular}{@{}lccc@{}}
      \toprule
      Method & PSNR$\uparrow$ & SSIM$\uparrow$ & LPIPS$\downarrow$ \\
      \midrule
      GTA\textsubscript{@188k} &  23.65 & 0.719 & 0.187\\
      GTA\textsubscript{@196k} &  23.60 & 0.716 & 0.184 \\
      GTA\textsubscript{@240k} & 23.46 & 0.712 & 0.196 \\
      GTA\textsubscript{@320k} & 23.39 & 0.710 & 0.201 \\
      PRoPE\textsubscript{@184k} & 23.67 & 0.715 & 0.195 \\
      PRoPE\textsubscript{@200k} & 23.61 & 0.719 & 0.197 \\
      PRoPE\textsubscript{@240k} & 23.61 & 0.716 & 0.199 \\
      PRoPE\textsubscript{@320k} & 23.14 & 0.700 & 0.215 \\
      DPPE\textsubscript{dual@320k}  & 23.92 & 0.727 & \textbf{0.176} \\
      DPPE\textsubscript{tAdd@320k}  & \textbf{23.98} & \textbf{0.728} & 0.180 \\
      \bottomrule
    \end{tabular}
\end{table}

\section{Additional Experiments}
\subsection{Comprehensive Dataset Extension and Generalization to Ray-Space Encodings}

\begin{table}[t]
    \centering
    \setlength{\fboxsep}{0pt}
    \setlength{\tabcolsep}{2pt}
    \caption{\textbf{Comparison of camera-based positional encodings across datasets.} We evaluate DPPE and PRoPE on various datasets. All values represent the mean over 3 seeds. For each metric, the best, second-best, and third-best results are highlighted in \colorbox{red!40}{red} and \colorbox{orange!30}{orange}, respectively, within each block separated by dashed lines.}
    \label{tab:extended_NVS_eval}
    \begin{NiceTabular}{ll c |[tikz=dashed] cccc |[tikz=dashed] ccc}
        \toprule
        Dataset & Metric & Pl\"ucker & GTA & PRoPE & DPPE\textsubscript{dual} & DPPE\textsubscript{tAdd} & UCPE & UCPE\textsubscript{dual} & UCPE\textsubscript{tAdd} \\
        \midrule
        \multirow{3}{*}{MVImgNet2}
        & PSNR$\uparrow$ & $23.07$ & $23.28$ & $22.91$ & \colorbox{orange!30}{$23.92$} & \colorbox{red!40}{${24.02}$}  & $23.27$ & \colorbox{orange!30}{$23.93$} & \colorbox{red!40}{$23.95$}\\
        & SSIM$\uparrow$ & $0.682$ & $0.706$ & $0.696$ & \colorbox{orange!30}{$0.726$} & \colorbox{red!40}{${0.730}$}  & $0.704$ & \colorbox{red!40}{$0.726$} & \colorbox{orange!30}{$0.725$}\\
        & LPIPS$\downarrow$ & $0.224$ & $0.203$ & $0.217$ & \colorbox{red!40}{${0.178}$} & \colorbox{orange!30}{$0.180$}  & $0.205$ & \colorbox{orange!30}{$0.179$} & \colorbox{red!40}{$0.176$}\\
        \midrule
        \multirow{3}{*}{RealEstate10K}
        & PSNR$\uparrow$ & $26.33$ & $24.15$ & $24.14$ & \colorbox{orange!30}{$24.16$} & \colorbox{red!40}{${24.27}$}  & $23.05$ & \colorbox{orange!30}{$23.66$} & \colorbox{red!40}{$23.75$}\\
        & SSIM$\uparrow$ & $0.849$ & $0.847$ & \colorbox{orange!30}{$0.851$} & $0.850$ & \colorbox{red!40}{${0.855}$}  & $0.831$ & \colorbox{orange!30}{$0.842$} & \colorbox{red!40}{$0.843$}\\
        & LPIPS$\downarrow$ & $0.135$ & $0.141$ & \colorbox{orange!30}{$0.137$} & $0.138$ & \colorbox{red!40}{${0.133}$}  & $0.158$ & \colorbox{red!40}{$0.146$} & \colorbox{red!40}{$0.146$}\\
        \midrule
        \multirow{3}{*}{SpatialVidHQ}
        & PSNR$\uparrow$ & $19.23$ & $19.00$ & $19.21$ & \colorbox{orange!30}{$19.26$} & \colorbox{red!40}{${19.28}$}  & \colorbox{orange!30}{$18.88$} & $18.76$ & \colorbox{red!40}{$19.13$}\\
        & SSIM$\uparrow$ & $0.610$ & $0.613$ & $0.626$ & \colorbox{orange!30}{$0.629$} & \colorbox{red!40}{${0.630}$}  & \colorbox{orange!30}{$0.612$} & $0.606$ & \colorbox{red!40}{$0.624$}\\
        & LPIPS$\downarrow$ & $0.278$ & $0.278$ & $0.268$ & \colorbox{orange!30}{$0.266$} & \colorbox{red!40}{${0.265}$}  & \colorbox{orange!30}{$0.286$} & $0.294$ & \colorbox{red!40}{$0.274$}\\
        \midrule
        \multirow{3}{*}{CO3D}
        & PSNR$\uparrow$ & $21.54$& {$20.06$} &$19.91$ & \colorbox{orange!30}{$20.55$} & \colorbox{red!40}{$20.81$} & {$20.56$} & \colorbox{orange!30}{$20.87$} & \colorbox{red!40}{$20.92$}\\
        & SSIM$\uparrow$ & $0.688$& {$0.689$} &$0.686$ & \colorbox{orange!30}{$0.702$}& \colorbox{red!40}{$0.709$} &$0.703$ & \colorbox{red!40}{$0.712$} & \colorbox{orange!30}{$0.711$}\\
        & LPIPS$\downarrow$ & $0.291$& {$0.290$} &$0.295$ & \colorbox{orange!30}{$0.256$} & \colorbox{red!40}{$0.245$} & {$0.260$} & \colorbox{red!40}{$0.244$} & \colorbox{orange!30}{$0.245$}\\
        \midrule
        \multirow{3}{*}{WildRGBD}
        & PSNR$\uparrow$ & $16.39$&$16.88$ & {$17.20$} & \colorbox{red!40}{$17.33$}& \colorbox{orange!30}{$17.27$} &$17.13$ & \colorbox{red!40}{$17.42$} & \colorbox{orange!30}{$17.23$}\\
        & SSIM$\uparrow$ & $0.526$&$0.553$ & {$0.564$} & \colorbox{red!40}{$0.570$} & \colorbox{orange!30}{$0.565$}&$0.558$ & \colorbox{red!40}{$0.571$} & \colorbox{orange!30}{$0.563$}\\
        & LPIPS$\downarrow$ & $0.422$&$0.363$ & \colorbox{orange!30}{$0.325$} & \colorbox{red!40}{$0.321$} & $0.338$&$0.362$ & \colorbox{red!40}{$0.317$} & \colorbox{orange!30}{$0.345$}\\
        \midrule
        \multirow{3}{*}{MegaSynth}
        & PSNR$\uparrow$ & $18.07$&$18.18$ & $18.20$& \colorbox{red!40}{$18.37$} & \colorbox{orange!30}{$18.33$} &$17.78$ & \colorbox{red!40}{$18.44$} & \colorbox{orange!30}{$18.31$}\\
        & SSIM$\uparrow$ & $0.513$&$0.538$ & $0.540$& \colorbox{red!40}{$0.549$} & \colorbox{orange!30}{$0.547$} &$0.512$ & \colorbox{red!40}{$0.556$} & \colorbox{orange!30}{$0.549$}\\
        & LPIPS$\downarrow$ & $0.354$&$0.335$ & $0.333$& \colorbox{red!40}{$0.320$} & \colorbox{orange!30}{$0.324$} &$0.371$ & \colorbox{red!40}{$0.317$} & \colorbox{orange!30}{$0.323$}\\
        \bottomrule
    \end{NiceTabular}
\end{table}

To verify the robustness and generalizability of the proposed decoupling framework, we extend our evaluation across two distinct axes: dataset diversity and alternative geometric encodings.

\paragraph{Evaluation Across Six Diverse Datasets} We evaluate the novel view synthesis (NVS) performance of DPPE against existing baselines across six datasets featuring highly varied scene structures, scale scales, and motion dynamics: MVImgNet2~\cite{mvimgnet2}, RealEstate10K~\cite{re10k}, SpatialVidHQ~\cite{spatialvid}, CO3D~\cite{co3d}, WildRGBD~\cite{wildrgbd}, and MegaSynth~\cite{megasynth}. This broad suite tests the models under both object-centric constraints (e.g., MVImgNet2, CO3D) and complex trajectory-based environments (e.g., RealEstate10K, SpatialVidHQ).

\paragraph{Generalization to Token-Level Ray Spaces (UCPE)} We investigate whether the benefits of decoupling apply beyond image-level camera matrices. We implement our decoupling strategy on Unified Camera Positional Encoding (UCPE)~\cite{ucpe}, a formulation that instantiates a ray space for each patch token rather than applying a global camera matrix per image. Following the principles detailed in \cref{sec:self-attn_dppe}, we construct UCPE\textsubscript{dual} and UCPE\textsubscript{tAdd} to assess if separating rotation and translation components resolves late-stage training degradation in token-level ray encodings.

\paragraph{Comparison with Explicit Geometric Baselines} For comprehensive reference, we compare these attention-based positional encodings against a non-attention baseline using Pl\"ucker coordinates. This baseline embeds local ray geometry directly into the token features without modifying the attention matrix via projection operations.

As reported in \cref{tab:extended_NVS_eval} and \cref{fig:extended_NVS}, DPPE\textsubscript{tAdd} and DPPE\textsubscript{dual} consistently outperform or match the best baseline performance across all datasets, validating the core identifiability analysis. Crucially, the performance gains observed when transitioning from UCPE to UCPE\textsubscript{dual} and UCPE\textsubscript{tAdd} mirror the improvements seen between PRoPE and DPPE. This demonstrates that architectural decoupling is a generalizable solution for stabilizing camera-based positional encodings, irrespective of whether the geometry is parameterized at the image level or the token level.

\newpage
\begin{figure}[H]
    \centering
    \begin{subfigure}[t]{0.95\textwidth}
        \includegraphics[width=\textwidth]{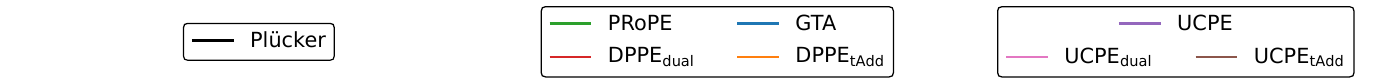}
        \label{fig:}
    \end{subfigure}
    \par
    \vspace*{-12pt}
    \begin{subfigure}[t]{0.95\textwidth}
        \includegraphics[width=\textwidth]{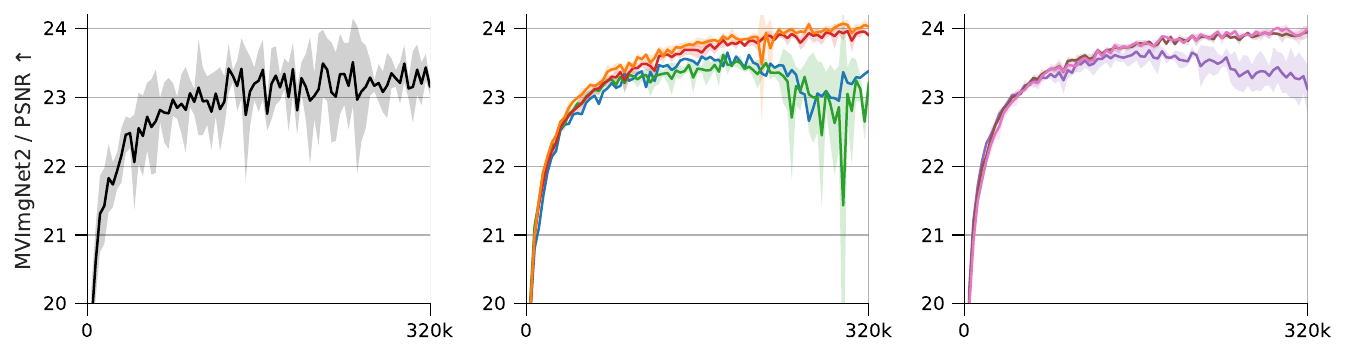}
        \label{fig:}
    \end{subfigure}
    \par
    \vspace*{-16pt}
    \begin{subfigure}[t]{0.95\textwidth}
        \includegraphics[width=\textwidth]{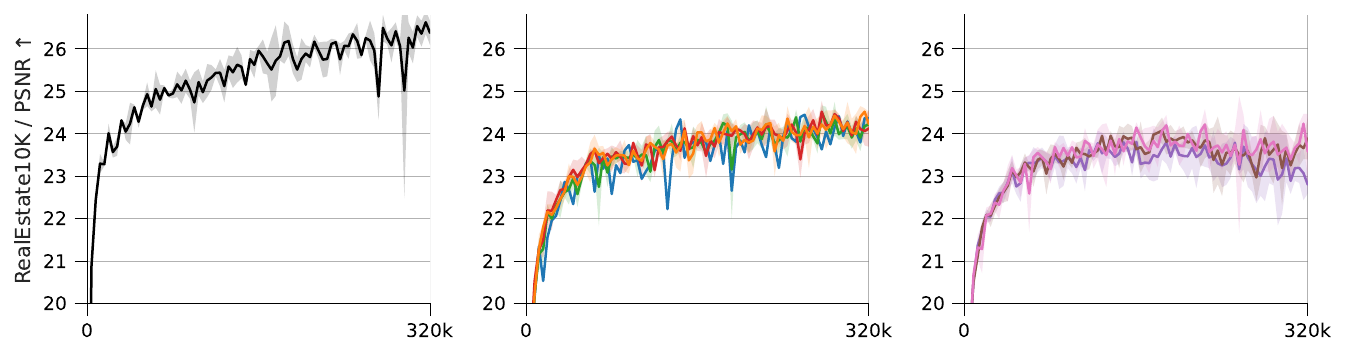}
        \label{fig:}
    \end{subfigure}
    \par
    \vspace*{-16pt}
    \begin{subfigure}[t]{0.95\textwidth}
        \includegraphics[width=\textwidth]{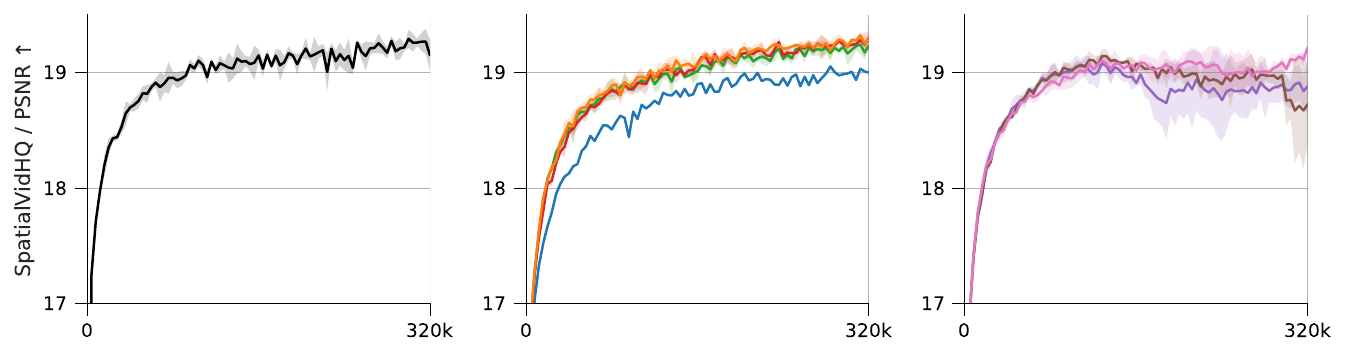}
        \label{fig:}
    \end{subfigure}
    \par
    \vspace*{-16pt}
    \begin{subfigure}[t]{0.95\textwidth}
        \includegraphics[width=\textwidth]{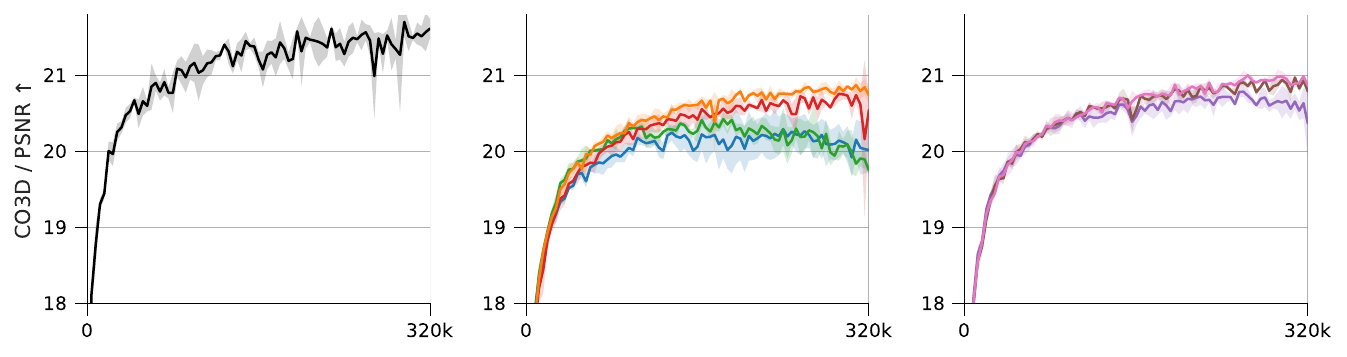}
        \label{fig:}
    \end{subfigure}
    \par
    \vspace*{-16pt}
    \begin{subfigure}[t]{0.95\textwidth}
        \includegraphics[width=\textwidth]{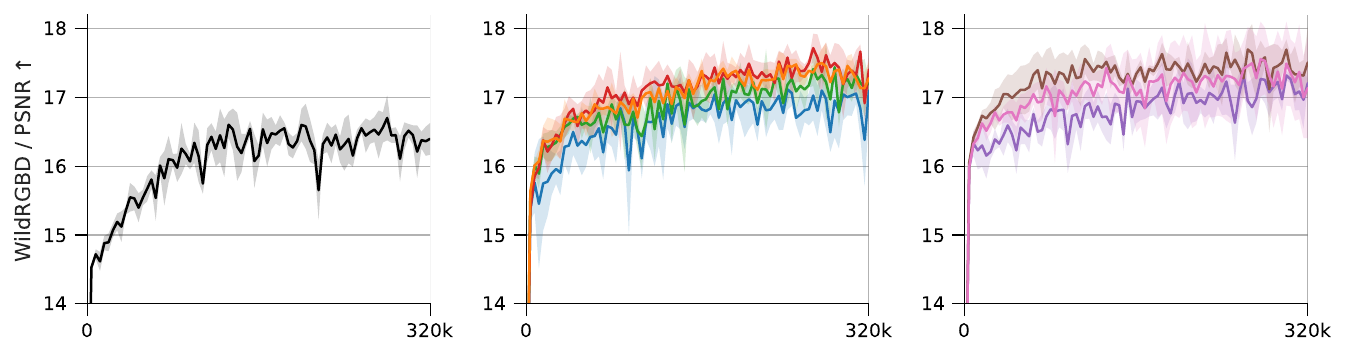}
        \label{fig:}
    \end{subfigure}
    \par
    \vspace*{-16pt}
    \begin{subfigure}[t]{0.95\textwidth}
        \includegraphics[width=\textwidth]{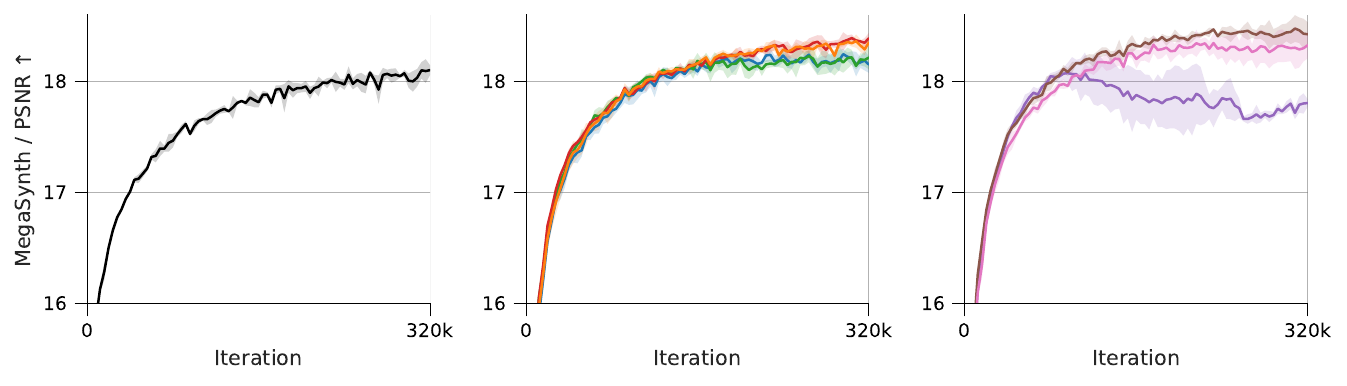}
        \label{fig:}
    \end{subfigure}
    \vspace*{-16pt}
    \caption{\textbf{Extended NVS performance (PSNR) across six datasets.} Results represent the mean over three seeds, and the shaded regions indicate the standard deviation. The learning curves show that our decoupled variants (\text{DPPE} and \text{UCPE's}) consistently achieve superior or competitive robustness and training stability compared to coupled baselines.}
    \label{fig:extended_NVS}
\end{figure}

\subsection{Large Scale Multi-View Depth Estimation}
\label{sec:depth_exp}

To demonstrate that DPPE achieves strong performance regardless of the target task, we apply and compare PRoPE~\cite{prope} and DPPE within VGGT~\cite{vggt}, a multi-view depth estimation model.

\paragraph{Experiment Details}
We conduct experiments using the VGGT~\cite{vggt} architecture. All model parameters are consistent with the original, and the pretrained DINOv2~\cite{DINOv2} encoder remains frozen during training. Training is performed for 240k iterations with a learning rate of $1.0 \times 10^{-4}$ and a warmup-stable-only scheduler (without decay). The minimum learning rate is set to $1.0 \times 10^{-6}$, with a warmup period of 1k iterations and a weight decay of 0.05. The training dataset consists of CO3D~\cite{co3d} (34,998 scenes), MegaSynth~\cite{megasynth} (65,200), WildRGBD~\cite{wildrgbd} (23,280), and ARKitScenes~\cite{arkitscenes} (4,498). For validation, we use the official train/val split for ARKitScenes. For CO3D and MegaSynth, as no official splits were available, we randomly sampled 2,000 scenes from each for validation. WildRGBD was excluded from the validation set due to significant missing pixels in its ground-truth depth, although it remains suitable for training.

\paragraph{Modified Depth Loss}
In the original VGGT, depth is predicted in a space normalized such that the mean depth of all valid pixels equals 1. However, since the model does not have prior information about valid pixel locations, we modified this normalization for consistency. In our setup, we set the reference camera to identity and normalize the mean distance of all camera centers to 1. Depth is then predicted at this scale and evaluated using a direct L1 loss. This approach provides a deterministic scale for both the data and the model. Additionally, we apply the four-stage multi-scale gradient loss used in the original VGGT. Camera loss and point tracking loss are omitted as they are not applicable to our specific experimental setting. Metrics in the plots are smoothed using a decay factor of 0.95 every 1k iterations.

\paragraph{Results and Discussion}
We compare four positional encoding methods: RoPE, which is used in original VGGT, PRoPE, DPPE\textsubscript{tAdd}, and DPPE\textsubscript{dual}. While stable evaluation on ARKitScenes is challenging due to its small proportion in the training data, consistent patterns emerge in CO3D and MegaSynth.

Consistent with the results on the similarly object-centric MVImgNet2~\cite{mvimgnet2} dataset shown in \cref{tab:main_train}, DPPE\textsubscript{dual} demonstrates steady performance improvements on CO3D even in the later stages of training. This confirms the effectiveness of DPPE\textsubscript{dual} not only for NVS tasks but also for multi-view depth estimation. Furthermore, in MegaSynth—a dataset characterized by significant viewpoint changes rather than object-centric scenes—the performance gains of PRoPE saturate during late-stage training, whereas DPPE\textsubscript{dual} maintains consistent improvement. We also observe that DPPE\textsubscript{tAdd} contributes significantly to performance enhancements in the final stages of MegaSynth training. These results indicate that DPPE generalizes well beyond object-centric scenarios, demonstrating robust performance on datasets involving simultaneous changes in rotation and translation.

\begin{table}[t]
    \centering
    \caption{
        Depth estimation performance measured by L1 loss.
        The best scores are shown in \textbf{bold}.
    }
    \label{tab:depth}
        \begin{tabular}{lccc}
            \toprule
            Method & ARKitScenesWAI & CO3DWAI & MEGASYNTHWAI \\
            \midrule
            RoPE      & 3.272 & 0.793 & 0.110 \\
            PRoPE     & 2.735 & 0.490 & 0.071 \\
            DPPE\textsubscript{tAdd}  & 2.827 & 0.499 & 0.070 \\
            DPPE\textsubscript{dual}  & \textbf{2.734} & \textbf{0.466} & \textbf{0.068} \\
            \bottomrule
        \end{tabular}
\end{table}

\subsection{Further Scale-Up Training}
\begin{figure}[t]
    \centering
    \includegraphics[width=\textwidth]{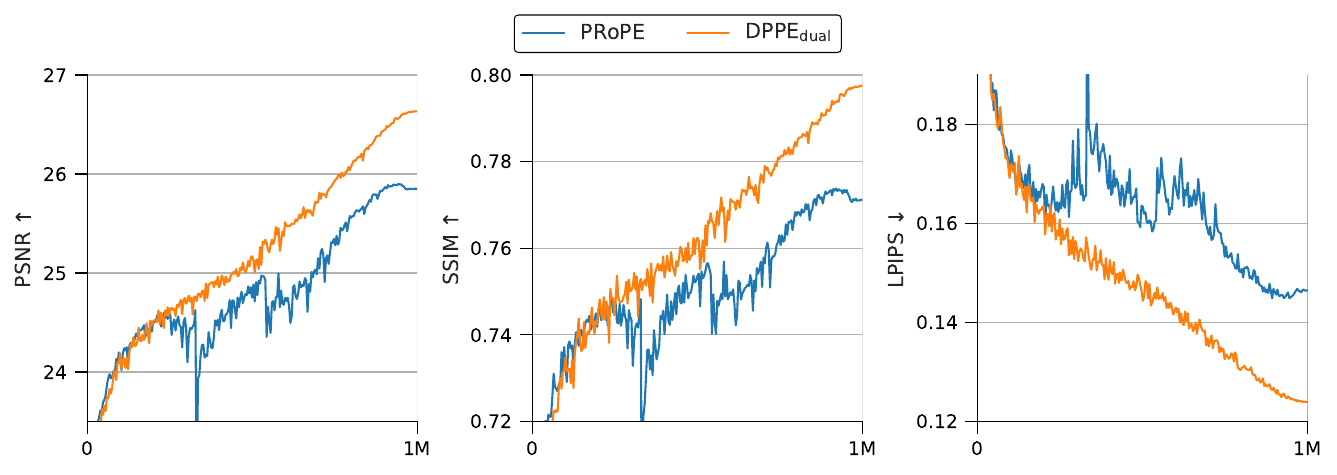}
    \caption{\textbf{Futher scaled-up training with PRoPE and DPPE.} DPPE shows stable training behavior throughout the entire schedule while PRoPE suffers from training instability.}
    \label{fig:large}
\end{figure}

To examine whether the benefit of DPPE persists in a more demanding training regime, we further scale up both the model size and the training horizon.
Specifically, we double the network depth from 12 to 24 layers and extend the training schedule from 320k to 1M iterations.
This setting is substantially more demanding than the configurations used in the preceding experiments, and at this scale training stability becomes a non-trivial concern.
To establish a reliable foundation for such long training runs, we additionally introduce a cosine decay learning-rate scheduler with a short linear warmup, which we find to be essential for stable convergence over the extended schedule.

We evaluate this scaled-up setting on MVImgNet2 for novel view synthesis (NVS), and compare DPPE against PRoPE~\cite{prope} under identical conditions. Following our base protocol, the number of input views is fixed to two.
The peak learning rate is set to 5e-4 with a linear warmup of 500 steps followed by cosine decay over the remaining iterations.
All other hyperparameters are kept identical to those used in \cref{sec:exp}.

As reported in \cref{fig:large}, DPPE exhibits stable training behavior throughout the entire schedule, with all evaluation metrics improving consistently up to the final iteration.
In contrast, PRoPE suffers from training instability and its performance degrades in the late stage of training.
This trend is consistent with the conclusion drawn in the main paper.

\begin{figure}[!htb]
    \centering
    \includegraphics[width=0.95\textwidth]{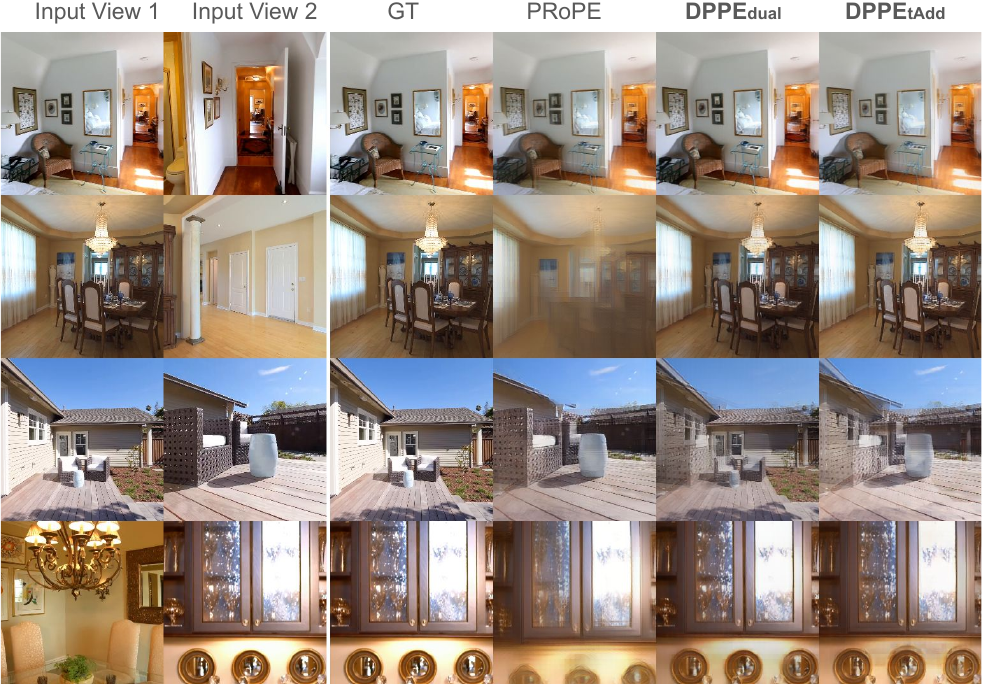}
    \caption{Qualitative comparison of DPPE and PRoPE for novel view synthesis on RealEstate10K.}
    \label{fig:qua-real}
\end{figure}

\begin{figure}[!htb]
    \centering
    \includegraphics[width=0.95\textwidth]{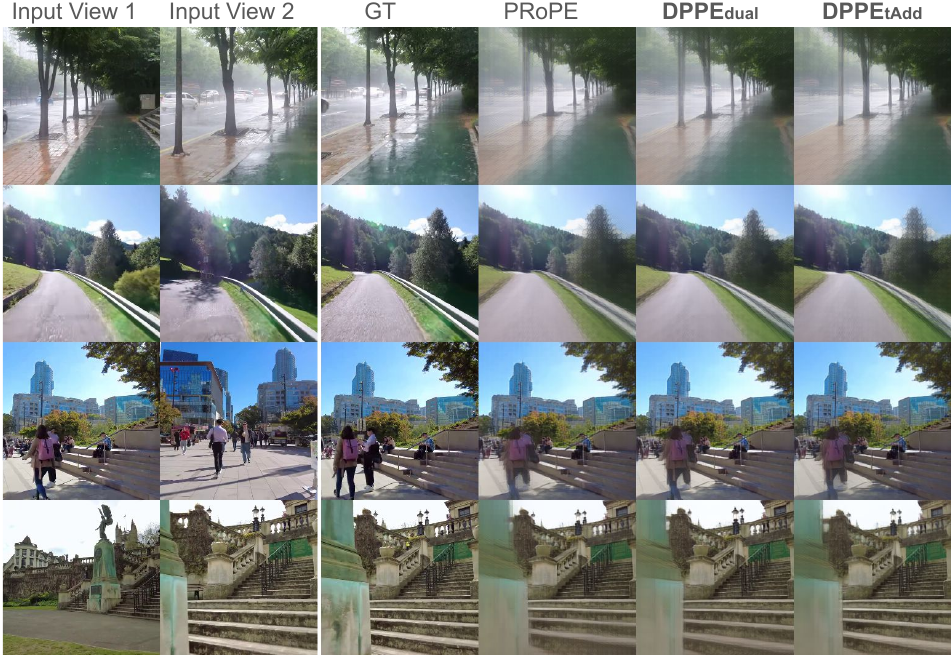}
    \caption{Qualitative comparison of DPPE and PRoPE for novel view synthesis on SpatialVidHQ.}
    \label{fig:qua-svid}
\end{figure}

\subsection{Qualitative Comparison on RealEstate10K and SpatialVidHQ}
\cref{fig:qua-real,fig:qua-svid} show the qualitative comparison of DPPE and PRoPE on RealEstate10K and SpatialVidHQ, respectively.